%% file: tacl2018v2.tex
\pdfoutput=1
\documentclass[11pt,a4paper]{article}
\usepackage{amsthm}
\newtheorem{theorem}{Theorem}[section]
\newtheorem{lemma}[theorem]{Lemma}

\usepackage{algorithm}
\usepackage[noend]{algpseudocode}

\usepackage{times,latexsym}
\usepackage{url}
\usepackage[T1]{fontenc}

\usepackage{makecell}

\usepackage{amsmath}

\usepackage[acceptedWithA]{tacl2018v2}

\usepackage{multirow}

\usepackage[]{tacl2018v2}

\usepackage{xspace,mfirstuc,tabulary}

\newif\iftaclinstructions
\taclinstructionsfalse % AUTHORS: do NOT set this to true
\iftaclinstructions
\renewcommand{\confidential}{}
\renewcommand{\anonsubtext}{(No author info supplied here, for consistency with
TACL-submission anonymization requirements)}
\newcommand{\instr}
\fi

\iftaclpubformat % this "if" is set by the choice of options

\else

\fi

\usepackage{amsmath}
\usepackage{amssymb}
\usepackage{graphicx}
\usepackage{makecell}
\usepackage{microtype}
\usepackage{tikz}
\usepackage{tikz-dependency}
\usetikzlibrary{shapes.geometric,calc}
\usetikzlibrary{colorbrewer}

\usepackage{tabularx}
\usepackage{booktabs}
\usepackage[normalem]{ulem}
\useunder{\uline}{\ul}{}
\usepackage{float}
\usepackage{pgfplots}
\pgfplotsset{compat=1.14, every non boxed x axis/.append style={x axis line style=-},
     every non boxed y axis/.append style={y axis line style=-}}

\title{Pruning the Index Contents for Memory Efficient Open-Domain QA}

\usepackage{lipsum}     % lorem ipsum
\usepackage{enumitem}% http://ctan.org/pkg/enumitem

\author{Martin Fajcik, Martin Docekal, Karel Ondrej, Pavel Smrz\\
  Brno University of Technology\\
  {\tt \{ifajcik,idocekal,ondrej,smrz\}@fit.vutbr.cz} }

\date{}

\begin{document}
\maketitle

\begin{abstract}
    % Present R2-D2 system
    This work presents a novel pipeline that demonstrates what is achievable with a combined effort of state-of-the-art approaches. Specifically, it proposes the novel R2-D2 (\textsc{Rank twice}, \textsc{reaD twice}) pipeline composed of retriever, passage reranker, extractive reader, generative reader and a simple way to combine them.
    % Present pruning
    Furthermore, previous work often comes with a massive index of external documents that scales in the order of tens of GiB. 
    This work presents a simple approach for pruning the contents of a massive index such that the open-domain QA system altogether with index, OS, and library components fits into 6GiB docker image while retaining only 8\% of original index contents and losing only up to 3\% EM accuracy\footnote{Our demo is available at \url{http://r2d2.fit.vutbr.cz/}. Code and preprocessed data are available at \url{https://github.com/KNOT-FIT-BUT/R2-D2}.}.
\end{abstract}

% \footnotetext[1]{equal contribution} % \thanks has some wierd formatting in TACL, I had to fallback to this workaround
%\renewcommand*{\thefootnote}{\arabic{footnote}}

% \section{Writing Notes/Guidelines}
% {\tiny
% \subsection{Contributions}
% Contributions in this paper are:
%     \begin{enumerate}
%         \item We present a pipelined approach which includes novel reranker approach, which works surprisingly well.
%         \item We present a new and very simple approach for pruning large corpus-based index via binary classifier without major performance loss (90\% size reduction and performance loss 3\%).
%         \item We observe an interesting trade-off between index pruning and reranker, their performance impact compensates each other (approximately).
%         \item Span-proposal for generative models helps a lot! Can it be included into T5.
%         \item Fusion matters! and its (almost) free!
%         \item Fine-tuning of retrieval components matters!
%     \end{enumerate}
    
\section{Introduction}
Recent advances in neural passage retrieval \cite[\textit{inter alia}]{karpukhin2020dense, izacard2020distilling, khattab2020relevance, luan2020sparse} greatly improved the performance of open-domain question answering systems (open-QA). 
The goal of these systems is to provide an answer to factoid questions. 
Traditional open-QA systems \cite{chen2017reading} seek evidence for answering these questions inside the knowledge source. 
This is often a large corpus of short snippets of natural language, so-called passages, with information-rich contents (e.g., taken from an encyclopedia). 
The current state-of-the-art systems can be scaled to millions or even billions \cite{seo2019real} of natural language passages. With the ongoing progress, and ever-growing sources of information, it can be expected that the open-QA will play a major role in everyday human life, e.g., in complementing or even replacing document search, as we know it \cite{etzioni2011search}. Therefore a natural question arises: \textit{Is all of this information relevant for current open-QA systems}? 

To gain evidence towards answering this question we experiment with our simple content-based pruning approach --- a binary classifier which selects whether the passage is irrelevant or not without seeing any question ---  on popular open-QA datasets NaturalQuestions, TriviaQA and EfficientQA. Surprisingly, we find that most (about 92\%) of the information content can be pruned away with only minor (3 EM) performance degradation to be seen in the current open-domain pipelined QA systems.

As our second contribution, we present a novel pipelined open-QA baseline composed of retriever, passage reranker, extractive reader, generative reader, and a simple component fusion approach.
Our system sets a new state-of-the-art on NaturalQuestions dataset. Furthermore it ended up among the top performing systems in the EfficientQA competition \cite{min2021neurips}\footnote{Leaderboard available at \url{https://efficientqa.github.io/}.}.

\section{Pruning Approach}
To reduce the size of the index, we resort to an apriori relevance classifier, which selects the relevant content without seeing a question. Note this is in contrast with the retriever, which considers a question when assigning the relevance. Consider the Wikipedia corpus split into 100-word passages. 
The recent work of \citet{karpukhin2020dense} indicates that the distribution of golden passages --- the passages containing an answer from the dataset --- differs from the distribution of all passages. 
This is implicated by the fact that golden passages perform as better negative samples than just any randomly sampled passages when training the retriever.
Therefore, given a passage $p_i$ from Wikipedia, we propose an apriori relevance classifier (we call \emph{pruner}) into relevance class $r$ that models the Bernoulli distribution $\boldsymbol{P}(r|p_i)$. 
The input of this classifier is the concatenation of Wikipedia passage (sometimes referred to as context) and its article's title separated with the special \texttt{SEP} token.
The classifier is trained via binary cross-entropy on the set of golden passages and non-golden passages extracted from Wikipedia.
In test-time, we collect the probabilities $\boldsymbol{P}(r|p_i)$ for each passage $p_i$ in the corpus. 
We keep only passages $p_i$ that satisfy the threshold constraint $\boldsymbol{P}(r|p_i)> \tau; \tau \in (0,1)$. 
Furthermore, we empirically verify in Section \ref{sec:results_and_analysis_embconnection} that the passage embeddings from \citet{karpukhin2020dense} contain strong features that capture the very same apriori relevance as this classifier does.

%%%%%%%%%%%%%%%%%%%%%%%%%%%%%%%%%%%%%%%%%%%%%%%%%%%%%%%%%%%%%%
% PIPELINE SECTION
%%%%%%%%%%%%%%%%%%%%%%%%%%%%%%%%%%%%%%%%%%%%%%%%%%%%%%%%%%%%%%
\section{Open-QA Pipeline}

To estimate the impact of corpus pruning on various open-QA components, we propose a pipelined system R2-D2 (\textsc{Rank twice}, \textsc{reaD twice}).  The parameters of each component are estimated separately. It is composed of DPR passage retriever \cite{karpukhin2020dense}, passage reranker (see subsection \ref{ss:reranker}), and two readers. Figure \ref{fig:r2d2_pipeline} shows the diagram of our system. The first reader performs an extractive span-selection similar to \citet{fajcik2020rethinking}. The second reader is based on Fusion-In-Decoder (FiD) \cite{izacard2020leveraging}.

Formally, given a question $q \in \mathcal{Q}$ from the set of all possible questions $\mathcal{Q}$  and the corpus $\mathcal{C}=\{p_1, p_2, ... , p_n\}$ composed of passages $p_i$, the retriever learns a ranking function $\operatorname{rank}:\mathcal{Q} \times \mathcal{C} \rightarrow \mathbb{R} $ that assigns a score to each passage. Note each passage contains its passage title.

Taking a top-$K$ scoring passages $\mathcal{C}_{r}$, reranker again re-scores $\mathcal{C}_{r}$ scoring passages by learning a reranking function $\operatorname{rerank}:\mathcal{Q} \times \mathcal{C}_{r} \rightarrow \mathbb{R}$. Note that while $\operatorname{rank}$ and $\operatorname{rerank}$ have similar signatures, the computational cost of $\operatorname{rerank}$ over the same amount of passages is drastically higher, as it computes fine-grained interaction between tokens of question and passage.

Next, we experiment with two readers: the extractive reader reads top-$V$ passages $\mathcal{C}_{rr}$ independently and assigns probability $\boldsymbol{P}_{e}(a_e|q, \mathcal{C}_{rr})$ to each span $a_e$ in the passages (see subsection \ref{ss:ext_reader}). 
The FiD generative reader reads top-$V_2$ passages $\mathcal{C}_{rr}'$ and generates an answer from probability space $\boldsymbol{P}_g(a_g|q,\mathcal{C}_{rr}')$ via greedy search.

Finally, R2-D2 aggregates the outputs from all components using two fusions (described in subsection \ref{ss:fusions}).

\begin{figure}[t!]
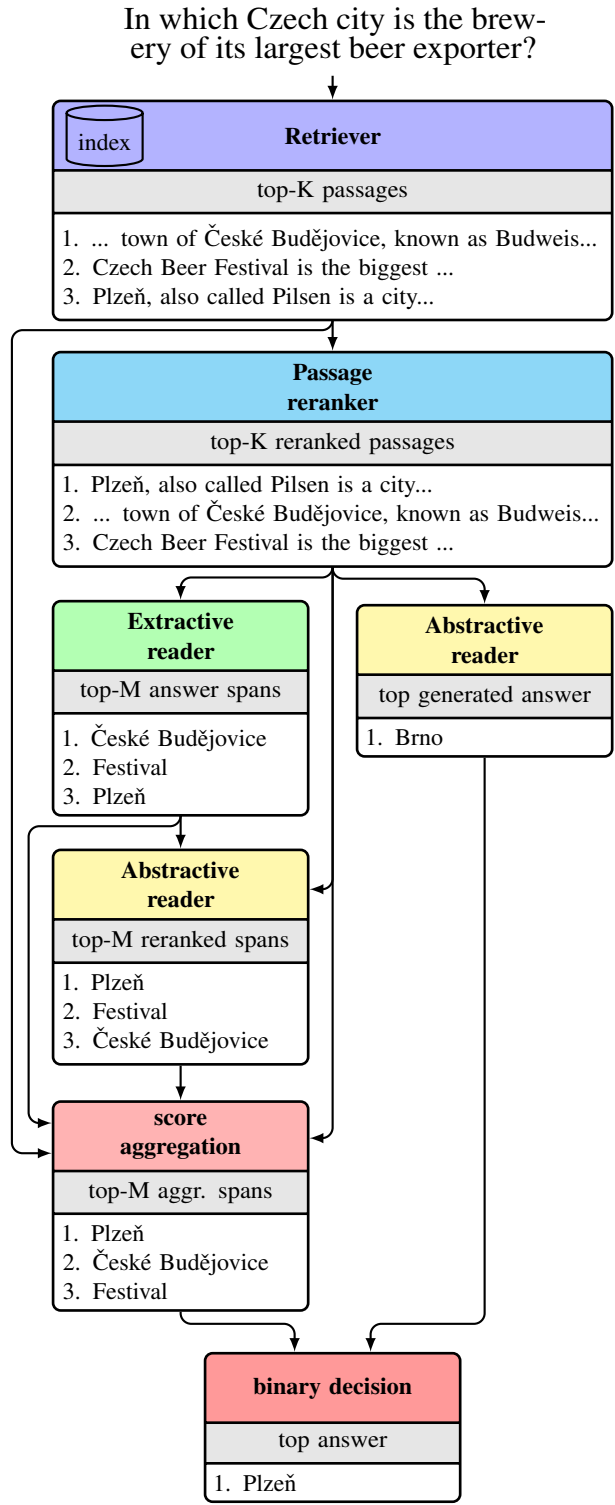

    \centering
    \include{figures/pipeline-v3}
    \caption{R2-D2 pipeline.}
    \label{fig:r2d2_pipeline}
\end{figure}

%%%%%%%%%%%%%%%%%%%%%%%%%%%%%%%%%%%%%%%%%%%%%%%%%%%%%%%%%%%%%%
% PASSAGE RERANKER (PIPELINE)
%%%%%%%%%%%%%%%%%%%%%%%%%%%%%%%%%%%%%%%%%%%%%%%%%%%%%%%%%%%%%%
\subsection{Passage Reranker}\label{ss:reranker}
The proposed passage reranker is based on transformer cross-encoder similar to \citet{nogueira2019passage, luan2020sparse}.
The input is the concatenation of question $q\in\mathcal{Q}$ and passage $p\in\mathcal{C}_r$ with a special \texttt{SEP} token between them. The passage consists of a title and context that are prepended with special start tokens and concatenated together.
We denote the contextual representation of input token $w$ obtained by the cross-encoder as $\operatorname{En}(p, q)[w]\in\mathbb{R}^d$.

%The loss function of the passage reranker is the cross-entropy.
    Now we can define the reranking function to the re-score passage as
\begin{equation}
    \operatorname{rerank}(q,p) = \operatorname{En}(p,q)[\texttt{CLS}]^\top w
\end{equation}
where $w\in\mathbb{R}^{d}$ is a trainable vector and \texttt{CLS} is the special token added at the start of an input sequence.
Finally, we define the following formula\footnote{Formal definition of softmax over a set is described in the Apendix \ref{app:softmax_not}.}
\begin{equation}
    \boldsymbol{P}_{rr}\left(p | q, C_r\right) =\operatorname*{softmax}\limits_{p\in \mathcal{C}_r}\left(\operatorname{rerank}\left(q, p\right)\right)_{p}
\end{equation}
to assign a probability to the case that passage $p$ contains answer to the question $q$.

%During the training, each input sequence contains exactly one ground truth passage and hard negative passages up to maximum length.
%Hard negatives are uniformly sampled from $\mathcal{C}_r$, which do not contain answer to the question.
\begin{description}[style=unboxed,leftmargin=0em,listparindent=\parindent]
    \setlength\parskip{0em}
\item[Training.] The model input for each question is exactly one positive sample supplemented with hard negatives from retriever. The ground truth passage, annotated the same way as in \citet{karpukhin2020dense}, is primarily used as a positive sample. If the ground truth is unknown, the positive sample is the best retriever passage containing the answer.
The hard negatives are uniformly sampled from retriever top-k results that do not contain the answer. 
The used loss function is the cross-entropy.
\end{description}

%%%%%%%%%%%%%%%%%%%%%%%%%%%%%%%%%%%%%%%%%%%%%%%%%%%%%%%%%%%%%%
% EXTRACTIVE READER (PIPELINE)
%%%%%%%%%%%%%%%%%%%%%%%%%%%%%%%%%%%%%%%%%%%%%%%%%%%%%%%%%%%%%%
\subsection{Extractive Reader}
\label{ss:ext_reader}

Extractive reader estimates the probability $\boldsymbol{P}_{e}(a_e|q, \mathcal{C}_{rr})$. 
It is the probability of a span $a_e$ from top-$V$ passage $p \in \mathcal{C}_{rr}$ being an answer to a question $q$. 
We decompose the $\boldsymbol{P}_{e}(a_e|q, \mathcal{C}_{rr})$ into four probabilities of:

\begin{itemize}
    \setlength\parskip{0em}
    \item token $s$ being starting token of an answer span,
    \item token $e$ being ending token of an answer span,
    \item tokens $s$ and $e$ being boundary tokens of an answer span \cite{fajcik2020rethinking},
    \item passage $p$ containing an answer for the question $q$ (inner reranker) as in \citet{karpukhin2020dense}.
\end{itemize}

% These probabilities are defined as:
% \begin{equation}
%     \boldsymbol{P}_{*}(*|q, \mathcal{C}_{rr}) =
%     \frac{e^{s_{*}}}{Z_*} \: ,
% \end{equation}
% where $*$ may stand for a \emph{start}, \emph{end}, \emph{joint}, and a \emph{passage}. $Z_*$ is the normalization sum that sums all element scores in all passages in $\mathcal{C}_{rr}$. On the other hand, the $s_*$ scores are estimated by the model with just a single passage on its input:

To obtain the final probability used in test-time, we multiply them all together\footnote{We tried decoding from the subsets of these probabilities in Appendix \ref{app:decoding_ext_probs} not observing significant difference.}. These probabilities are defined as:
\begin{equation}
    \boldsymbol{P}_{*}(*|q, \mathcal{C}_{rr}) = \operatorname{softmax}(s_*)_i \: ,
\end{equation}
where $*$ may stand for a \emph{start}, \emph{end}, \emph{joint}, and a \emph{passage}. The $i$ is an index of a given element, and the $s_*$ is a vector of scores for each element among all passages in $\mathcal{C}_{rr}$. So the \emph{softmax} normalization sum goes through all the passages. On the other hand, the $s_*$ scores are estimated by the model with just a single passage on its input \cite{clark-gardner-2018-simple}. The scores are as follows:
\setlength{\jot}{1ex}
\begin{gather}%
s^{i}_{start} = \operatorname{En}(p,q)[s]^\top w_{start} \\
s^{i}_{end} = \operatorname{En}(p,q)[e]^\top w_{end} \\
s^{i}_{joint} = (W_j \operatorname{En}(p,q)[s] + b_j)^\top \operatorname{En}(p,q)[e] \\
s^{i}_{passage} = \operatorname{En}(p,q)[\texttt{CLS}]^\top w_{p} \:.
\end{gather}%
Where $w_*,b_j \in \mathbb{R}^h$, $\operatorname{En}(p, q)[\cdot] \in \mathbb{R}^h$, and $W_j \in \mathbb{R}^{h \times h}$ are all trainable.

We omit the spans of a title and question for answer span selection. Therefore the final answer can be selected only from the context.

The following training objective with independently marginalized components is used:
\begin{equation} \label{eq:extReaderIndLoss}
\begin{split}
	-\log \sum_{s \in starts(C_{rr})} \boldsymbol{P}_{start}(s|q, \mathcal{C}_{rr}) \\
	-\log \sum_{e \in ends(C_{rr})} \boldsymbol{P}_{end}(e|q, \mathcal{C}_{rr}) \\
	-\log \sum_{j \in boundaries(C_{rr})} \boldsymbol{P}_{joint}(j|q, \mathcal{C}_{rr}) \\
	-\log \sum_{p \in C_{rr}} \boldsymbol{P}_{passage}(p|q, \mathcal{C}_{rr}) \: .
\end{split}
\end{equation}

The sums are going through target annotations (starts, ends, etc.) obtained by the distant supervision approach.
\subsection{Component Fusion}
\label{ss:fusions}
To produce the final answer, R2-D2 aggregates the log-probabilities of all system components via linear combinations tuned on validation data.  

Firstly, the log-probabilities of all system components for top-$M$ answer spans proposed by the extractive reader are aggregated. Formally, assume the $\mathcal{A}_q$ is the set of top-M answer spans from $\boldsymbol{P}_{e}(a|q,\mathcal{C}_{rr})$ for question $q$.
The generative model performs  the \textbf{answer reranking} evaluating the log-probability of the answer spans 
\begin{equation}
    \label{eq:genrerank}
    \{\log\boldsymbol{P}_g(a|q,\mathcal{C}_{rr}'): a\in \mathcal{A}_q\}.
\end{equation}

Next a logistic regression loss \eqref{eq:aggloss} is minimized to perform \textbf{score aggregation}. It combines the scores across the R2-D2 components to maximize the correct answer span probability over dataset~$\mathcal{D}$. This dataset is composed of the top-$M$ outputs of the extractive reader with the correct answer.
\begin{gather}%
% x(a) =[\boldsymbol{P}_{e}(a);\boldsymbol{P}_g(a); \boldsymbol{P}_r(p_a); \boldsymbol{P}_{rr}(p_a)]
% x(a) =
% \begin{bmatrix} 
% \boldsymbol{P}_{e}(a) & \boldsymbol{P}_g(a) & \boldsymbol{P}_r(p_a) & \boldsymbol{P}_{rr}(p_a)
% \end{bmatrix}
x(a) =[\boldsymbol{P}_{e}(a) \; \boldsymbol{P}_g(a) \; \boldsymbol{P}_r(p_a) \; \boldsymbol{P}_{rr}(p_a)] \\
\label{eq:aggloss}
   -\sum_{(\mathcal{A}_q,gt) \in \mathcal{D}} \operatorname*{softmax}\limits_{a \in \mathcal{A}_q} \big({ w^\top \log x(a) + b}\big)_{gt}
\end{gather}%
Here $p_a$ denotes the passage containing the answer span $a$, $\mathcal{A}_q$ is a set of proposed answer spans, $gt$ is the correct answer span, distribution dependencies are dropped for clarity and only the logistic regression parameters $w, b$ are tuned in this step.

Finally, we theorized the correct answer span might not always be available in the passage set $\mathcal{C}_{rr}$, but the generative reader might be able to generate the answer from its parameters and the evidence given in passages. We introduce the binary classifier, which decides whether to select the best span answer from answer aggregation step or a free-form answer generated via FiD. Given that $s_{agg}(q)=\max_{a\in\mathcal{A}_q} w^\top x(a)+b$ is the best span score and $s^*_g(q)=\log\boldsymbol{P}_g(a_q^*|q,\mathcal{C}_{rr}')$ is the log-probability of the answer $a_q^*$  obtained via greedy decoding for question $q$, a classifier is trained via binary cross-entropy $BCE(l,t)$ with log-odds ratio $l$ and target $t$ to do the \textbf{binary decision}
\begin{equation}
\label{eq:bdformula}
       \sum_{(e,t) \in \mathcal{D}} BCE(w^\top [s_{agg}(e);s^*_g(e)]+b, t ).
\end{equation}
Here, the training dataset $\mathcal{D}$ contains only cases where either the extractive or the abstractive prediction is correct (but not both).
%%%%%%%%%%%%%%%%%%%%%%%%%%%%%%%%%%%%%%%%%%%%%%%%%%%%%%%%%%%%%%
% EXPERIMENTAL SETUP SECTION
%%%%%%%%%%%%%%%%%%%%%%%%%%%%%%%%%%%%%%%%%%%%%%%%%%%%%%%%%%%%%%
\section{Experimental Setup}
We implement models in PyTorch \cite{paszke2019pytorch} using Transformers \cite{wolf-etal-2020-transformers}. We use 12GB GPU to train the passage reranker, 48GB GPU for the generative reader, and 16x 32GB GPUs to train the extractive reader with $V=128$ passages at its input. The inference runs on 12GB GPU. In all experiments, we used Adam optimizer with a decoupled weight decay \cite{loshchilov2017decoupled}. Our models are evaluated by two metrics:
\begin{description}[style=unboxed,leftmargin=0em,listparindent=\parindent]    \setlength\parskip{0em}
    \item [Exact match (EM)] measures the proportion of examples, for which the system prediction matched at least one annotated ground-truth answer. We use the script from \citet{lee-etal-2019-latent}\footnote{\url{https://cutt.ly/rkZNIer}}.
    \item [Accuracy@K] measures the proportion of examples, for which the ground-truth answer string is present in top-K retrieved passages. We match the string exactly as \citet{karpukhin2020dense}\footnote{\url{https://cutt.ly/0luNhx4}}.
\end{description}

\subsection{Datasets and Data Pre-processing}

We evaluate our models on three datasets. Their statistics are available in Table \ref{fig:datatsets}. To train the reranker we filter out examples, which do not contain golden passage or exact match in top-$K$ retrieved passages. 
To train the extractive reader, only examples with exact match in golden passage or top-1 retrieved passage are kept. Both filtering strategies are closely described in Appendix \ref{app:data_preprocessing}.

\begin{description}[style=unboxed,leftmargin=0em,listparindent=\parindent]
\setlength\parskip{0em}
\item[NQ-Open] \cite{kwiatkowski2019natural, lee-etal-2019-latent} or NaturalQuestions-Open, consists of real user queries obtained from Google search engine. The maximum length of each answer is at most 5 tokens. Each training and development sample contains 1 annotated answer, while test data contain 5-way answer annotation.

\item[TQ-Open] \cite{joshi-etal-2017-triviaqa} or TriviaQA-Open consists of question-answer pairs from 14 different trivia quiz websites. Each question contains human annotated answer and a set of answer aliases gathered from Wikipedia. We use the unfiltered version.

\item[EfficientQA] \cite{min2021neurips} is a dataset collected the same way as NQ-Open through 2019, and thus may contain more questions without evidence in our corpus than NQ-Open. Furthermore, it doesn't suffer from dev/test discrepancy, as it was collected for open-domain QA directly (see Appendix B in \citet{min2020ambigqa}). We use the officially released dev set for testing.

\end{description}

%%%%%%%%%%%%%%%%%%%%%%%%%%%%%%%%%%%%%%%%%%%%%%%%%%%%%%%%%%%%%%
% DATASET STATISTICS TABLE
%%%%%%%%%%%%%%%%%%%%%%%%%%%%%%%%%%%%%%%%%%%%%%%%%%%%%%%%%%%%%%
\begin{table}
    \centering
    \scalebox{0.92}{\input{tables/revision/dataset_statistics}}
    \caption{Dataset statistics. The filt. lines report how many examples are kept for training the reranker (filt. reranker) and extractive reader (filt. ext. reader). The lines w/ golden passage denote how many examples from the set contain golden passage annotation. The \textit{Golden} sets are a datasets used to estimate the pruner. }
    \label{fig:datatsets}
\end{table}

%%%%%%%%%%%%%%%%%%%%%%%%%%%%%%%%%%%%%%%%%%%%%%%%%%%%%%%%%%%%%%
% MODEL AND PIPELINES (EXPERIMENTAL SETUP)
%%%%%%%%%%%%%%%%%%%%%%%%%%%%%%%%%%%%%%%%%%%%%%%%%%%%%%%%%%%%%%
\subsection{Models and Pipeline}

\begin{description}[style=unboxed,leftmargin=0em,listparindent=\parindent]
\setlength\parskip{0em}

%%%%%%%%%%%%%%%%%%%%%%%%%%%%%%%%%%%%%%%%%%%%%%%%%%%%%%%%%%%%%%
% PRUNER (EXPERIMENTAL SETUP)
%%%%%%%%%%%%%%%%%%%%%%%%%%%%%%%%%%%%%%%%%%%%%%%%%%%%%%%%%%%%%%
\item[Pruner and Pruning.]
 We fine-tune the base version of ELECTRA \cite{Clark2020ELECTRA:} with a 2-layer feed-forward network on top of it (the same way as authors do it in classification tasks) as binary classifier. 
To train the pruner, we create training set with 2 negative passages per positive passage from dataset's training examples with golden passage annotation.  The negative passages are uniformly sampled from all non-golden Wikipedia's passages.
To create development and test sets for pruner, we split the subset of dataset's development set, with examples containing golden passage annotation, using a $1:2$ ratio. We sample only one negative passage per positive sample for development and test sets so that datasets are balanced. We further refer to these datasets as \emph{Golden}. The procedure is same for both datasets. The system is trained via cross-entropy in 2 epochs using batch size $12$ and learning rate $3 \cdot 10^{-5}$ linearly decreasing to~0.
The $\tau$ threshold is tuned so that we pool top 1.7M passages to fit the 6GiB limit.
We combine these relevant passages with missing golden passages from the training data, obtaining 1,702,133 passages in total for NQ-Open and EfficientQA and 1,706,676 passages for TQ-Open.

%%%%%%%%%%%%%%%%%%%%%%%%%%%%%%%%%%%%%%%%%%%%%%%%%%%%%%%%%%%%%%
% RETRIEVER (EXPERIMENTAL SETUP)
%%%%%%%%%%%%%%%%%%%%%%%%%%%%%%%%%%%%%%%%%%%%%%%%%%%%%%%%%%%%%%
\item[Retriever.]
We use BERT-based DPR from the official checkpoint\footnote{\url{https://github.com/facebookresearch/DPR}}. Each passage is represented via 768-dimensional embedding. We use multiset checkpoint for TQ-Open, as the checkpoint for TQ directly isn't officialy released.
We use the same knowledge corpus containing 21,015,320 passages based on 12-20-2018 Wikipedia snapshot as \citet{karpukhin2020dense}. In inference time, the retriever passes $K=200$ passages $\mathcal{C}_r$ to reranker.

%%%%%%%%%%%%%%%%%%%%%%%%%%%%%%%%%%%%%%%%%%%%%%%%%%%%%%%%%%%%%%
% PASSAGE RERANKER (EXPERIMENTAL SETUP)
%%%%%%%%%%%%%%%%%%%%%%%%%%%%%%%%%%%%%%%%%%%%%%%%%%%%%%%%%%%%%%
\item[Passage reranker.]
We use the RoBERTa-base \cite{liu2019roberta} and truncate the inputs to maximum length 256. The linear scheduler with 0.1 warmup proportion is used, the number of epochs is 5 and the model is validated every 40,000 optimization steps. 
The initial learning rate is $1.6\cdot10^{-4}$, batch size equals to 8 and model reranks 24 passages per question from top-400 DPR retrieved passages.
During the inference, top-$K$ retriever passages are rescored and passed to readers.

%%%%%%%%%%%%%%%%%%%%%%%%%%%%%%%%%%%%%%%%%%%%%%%%%%%%%%%%%%%%%%
% EXTRACTIVE READER (EXPERIMENTAL SETUP)
%%%%%%%%%%%%%%%%%%%%%%%%%%%%%%%%%%%%%%%%%%%%%%%%%%%%%%%%%%%%%%
\item[Extractive reader.] The extractive reader encoder is based on pre-trained ELECTRA-large. 
Its inputs are truncated if they are longer than the allowed maximum size (512 tokens). 
During the training phase, all spans from all $p \in \mathcal{C}_{r}$\footnote{Note that we use the retriever output directly.} that match\footnote{Matching strategies are described in Appendix \ref{app:data_preprocessing}.} with at least one of the known answers are selected as target annotations. 
Therefore the annotations might appear in the wrong context.

The extractive reader reads top 128 passages during the training phase and when it is used without the reranker. 
To demonstrate the effect of reranker, the reader reads only top 24 passages if the reranker is used.

We used a linear scheduler with a warmup for the first 20,000 steps for all models. 
The maximum number of training steps was 200,000. 
The model was validated every 20,000 steps, and the best checkpoint among validations was selected. 
The initial learning rate was $2 \cdot 10^{-5}$ and the optimization step was done after each training example.

%%%%%%%%%%%%%%%%%%%%%%%%%%%%%%%%%%%%%%%%%%%%%%%%%%%%%%%%%%%%%%
% GENERATIVE READER (EXPERIMENTAL SETUP)
%%%%%%%%%%%%%%%%%%%%%%%%%%%%%%%%%%%%%%%%%%%%%%%%%%%%%%%%%%%%%%
\item[Generative reader.]
We utilize T5-large \cite{raffel2020exploring} and use a concatenation of question, passages and their respective titles at the Fusion-in-Decoder's input the same way as \citet{izacard2020distilling}.  
We truncate each passage to length 250 tokens for NQ. 
For TQ, as questions are significantly longer, we truncate whole inputs to the same size.
Following FiD for TQ, we use only human-generated answer. 
In training, the golden passage always comes first, if available, and we take the rest of passages as ranked from previous step up to $V_2$ passages.
Due to the large memory requirements of the original approach, we use only $V_2 = 25$ passages. 
We use the similar hyperparameters as the original work --- batch size 64, learning rate $5 \cdot 10^{-5}$ but no learning rate schedule. 
In test time, we decode an answer via greedy decoding.
\end{description}

%%%%%%%%%%%%%%%%%%%%%%%%%%%%%%%%%%%%%%%%%%%%%%%%%%%%%%%%%%%%%%
% COMPRESSING THE IMAGE SIZE (EXPERIMENTAL SETUP)
%%%%%%%%%%%%%%%%%%%%%%%%%%%%%%%%%%%%%%%%%%%%%%%%%%%%%%%%%%%%%%
\subsection{Compressing the image size}
We save models and index in half-precision without significant loss of performance. Furthermore, we use off-the-shelf \texttt{ZIP}\footnote{\url{https://launchpad.net/ubuntu/+source/zip}} compression to reduce the size of the models and the corpus.
To fit the 6GiB limit, we use 100MB CentOS8 docker image\footnote{\texttt{nvidia/cuda:10.2-base-centos8}} and we also compress python's \texttt{site-packages} to reduce the size of PyTorch.

%%%%%%%%%%%%%%%%%%%%%%%%%%%%%%%%%%%%%%%%%%%%%%%%%%%%%%%%%%%%%%
% RESULTS AND ANALYSES
%%%%%%%%%%%%%%%%%%%%%%%%%%%%%%%%%%%%%%%%%%%%%%%%%%%%%%%%%%%%%%
\section{Results and Analysis}

%%%%%%%%%%%%%%%%%%%%%%%%%%%%%%%%%%%%%%%%%%%%%%%%%%%%%%%%%%%%%%
%  SYSTEMS (RESULTS AND ANALYSES)
%%%%%%%%%%%%%%%%%%%%%%%%%%%%%%%%%%%%%%%%%%%%%%%%%%%%%%%%%%%%%%
% Please add the following required packages to your document preamble:
% \usepackage{multirow}
% Please add the following required packages to your document preamble:
% \usepackage{multirow}
\begin{table}[t]
    \scalebox{0.69}{\input{tables/revision/systems}}
    \caption{Comparison with the state-of-the-art in EM. \#$\theta$ denotes the estimated amount of model parameters. Symbol $^*$ denotes systems with pruned or compressed index. Symbol $^{-}$ reports the result only for smaller system with $220M$ parameters.}
    \label{tab:systems}
\end{table}

\begin{description}[style=unboxed,leftmargin=0em,listparindent=\parindent]
\setlength\parskip{0em}
\item[Overall results.] The effectiveness of our approach is compared with the state-of-the-art in Table \ref{tab:systems}. 
Our system composed of just the retriever and FiD reader R1-D1 (Generative) shows inferior performance compared to FiD-large. 
This is most likely caused by 4 times fewer passages at its input, as in \citet{izacard2020leveraging}. 
In contrast, our ELECTRA based extractive reader R1-D1 (Extractive) shows large gains compared to extractive state-of-the-art, while having the same retriever as DPR. 
We hypothesize this may be caused by its better pre-training method, which shows strong performance through variety of tasks, but also due to training and inference with extra large input size of 128 passages and better objective. 
Finally, notice that our pruned system R2-D2 (1.7M) is competitive with FiD even when using just 1.7M knowledge corpus, and our full system R2-D2 (21M) is competitive even with FiD++, which uses DPR retriever improved via knowledge distillation and 26M passage corpus which also includes lists. 
Additionally, we evaluate our model with better retrieval model (HN-DPR) based on DPR checkpoint where hard negatives are mined using the retrieval model itself\footnote{\url{https://cutt.ly/Ux5Yt4h}}.

%%%%%%%%%%%%%%%%%%%%%%%%%%%%%%%%%%%%%%%%%%%%%%%%%%%%%%%%%%%%%%
%  ABLATION STUDY TABLE (RESULTS AND ANALYSES)
%%%%%%%%%%%%%%%%%%%%%%%%%%%%%%%%%%%%%%%%%%%%%%%%%%%%%%%%%%%%%%
% \begin{table*}[t]
%     \centering
%     \scalebox{0.93}{\input{tables/ablation-study}}%
%     \caption{Ablation study. The $\Delta$ column shows the exact match difference caused by pruning.}
%     \label{tab:ablation_study}
% \end{table*}
\begin{table*}[t]
    \centering
    %\scalebox{0.63}{\input{tables/revision/ablation-study}}%
    \scalebox{0.76}{\input{tables/revision/ablation-study-wo-trivia-dev}}%
    %\scalebox{0.90}{\input{tables/revision/ablation-study-wo-nq-open-dev-and-trivia-dev}}%
    \caption{Ablation study. The $\Delta$ column shows the exact match difference caused by pruning.}
    \label{tab:ablation_study}
\end{table*}

\usepgfplotslibrary{colorbrewer}
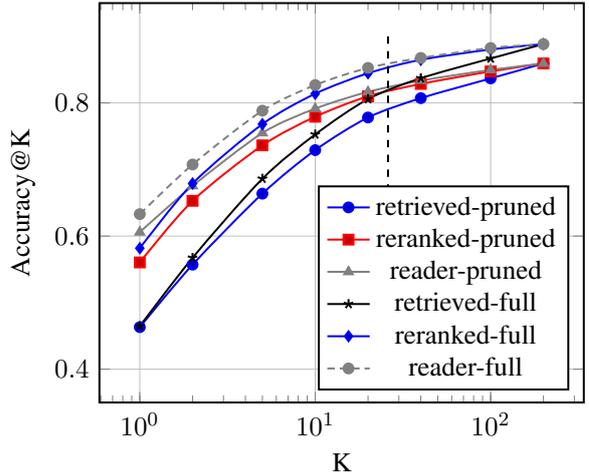
\begin{figure}[t]
    \centering
    \scalebox{0.93}{\input{charts/accuracy-at-k_nq-open-test}}%
    \caption{
    \label{fig:acc_k_nq}Accuracy@K on test-data of NQ-Open.}%
\end{figure}%
\item[Reranker performance.]
Next, we analyze the performance of our retriever and reranker with Accuracy@K in Figure \ref{fig:acc_k_nq}. 
The reranker improves the accuracy consistently for both, pruned and full version of our pipeline. 
Remarkably, the pruned version of our pipeline with reranker (reranked-pruned) performs better than the full version only with retriever (retrieved-full) up to $K=26$ paragraphs. 
We observe the similar trend on other datasets, e.g. for TQ-Open test the reranked-pruned improves over retrieved-full up to $K=116$ paragraphs (the analyses are in Appendix \ref{app:accuracy_at_k}).
We also include analysis, where we rerank each passage $p_i$ according its $s^i_{passage}$ score from extractive reader.
We observe results similar to reranker for $K>10$, indicating the reader reranks well on its own. 

\item[Pruner.] Our simple pruning approach achieved $90.63$\% accuracy on NQ-Golden test data and $86.94$\% accuracy on TQ-Golden test data. 
This indicates that there exists a strong prior over the passages of Wikipedia in these open-domain QA datasets. 
Interestingly, pruner still missed 2,133/40,670 (5.2\%) golden passages from the NQ-Golden training data and 6,676/50,502 (13.2\%) from the TQ-Golden training data.

%%%%%%%%%%%%%%%%%%%%%%%%%%%%%%%%%%%%%%%%%%%%%%%%%%%%%%%%%%%%%%
%  MEMORY FOOTPRINT (RESULTS AND ANALYSES)
%%%%%%%%%%%%%%%%%%%%%%%%%%%%%%%%%%%%%%%%%%%%%%%%%%%%%%%%%%%%%%
\item[Memory footprint.] 
Furthermore, we compare the memory footprint of our pruned and compressed system's docker image (pruned system) with the image of the full system on NQ-open in Figure~\ref{char:component_sizes}. 
The total uncompressed size of an image is 81.01GiB while the size of the pruned image is 5.96GiB (92.6\% less). 
Here, \emph{codes} are python code and configurations, \emph{corpus} is an sqlite3 database of passages, and \emph{binaries} are the OS with python libraries. 
We save \emph{dense index} as a raw \texttt{h5} matrix. Interestingly, the dense corpus has a similar space requirements as the \emph{parameters} of all 4 models used in this work.

%%%%%%%%%%%%%%%%%%%%%%%%%%%%%%%%%%%%%%%%%%%%%%%%%%%%%%%%%%%%%%
%  COMPONENT SIZE TABLE (RESULTS AND ANALYSES)
%%%%%%%%%%%%%%%%%%%%%%%%%%%%%%%%%%%%%%%%%%%%%%%%%%%%%%%%%%%%%%
% \begin{figure}
% \hspace*{-0.3cm}% Little bit of hacking to slightly increase image size
% \includegraphics[width=0.52\textwidth, angle=0]{figures/sizes.pdf}
% \caption{Component sizes.}
% \end{figure}
% \begin{figure}[ht]
%     \centering
%     \scalebox{0.85}{\input{charts/component-size}}%
%     \caption{Component sizes inside the docker image.}%
%     \label{char:component_sizes}
% \end{figure}%
\begin{figure}
    \centering
    \scalebox{0.88}{\input{charts/component-size-v2}}%
    \caption{Component sizes inside the docker image.}%
    \label{char:component_sizes}
\end{figure}
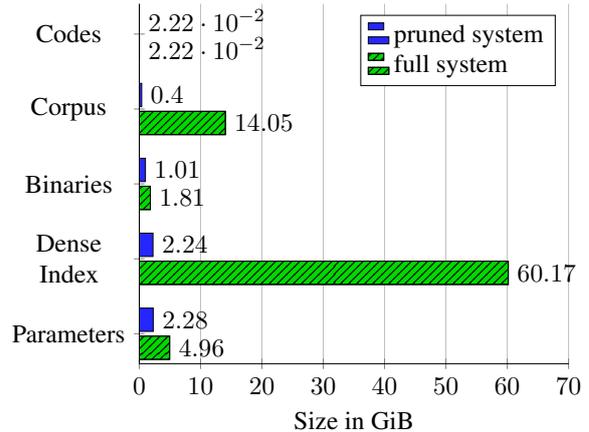%
%%%%%%%%%%%%%%%%%%%%%%%%%%%%%%%%%%%%%%%%%%%%%%%%%%%%%%%%%%%%%%
%  ABLATIONS (RESULTS AND ANALYSES)
%%%%%%%%%%%%%%%%%%%%%%%%%%%%%%%%%%%%%%%%%%%%%%%%%%%%%%%%%%%%%%
\item[Ablations.] The ablations are listed in Table~\ref{tab:ablation_study}. 
We ablate results with and without using passage reranker (first column), with separate readers and their combination (second column) and with different stages of component fusion (third column). 
Namely, performing a \emph{naive} answer re-ranking by generative reader means the system chooses the most probable answer span among the top-$M$ spans provided by the extractive reader according to generative reader log-probabilities as shown in equation \eqref{eq:genrerank}. 
Analogously, the \emph{aggr} fusion denotes that the system chooses the most probable answer span according to aggregated scores, as in equation \eqref{eq:aggloss}. 
Finally, the \emph{aggr+bd} fusion denotes the binary decision, as shown in equation \eqref{eq:bdformula}. 

As expected, we observe that reranker improves the results consistently for generative model in all but one case (NQ-Open (1.7M)). 
The gains are especially large for TQ-Open (over 3 EM, underscored in Table). 
In fact, the results are very close or better to \citet{izacard2020leveraging}, suggesting that using the FiD reader with smaller context window and reranker is a reasonable alternative to memory inefficient FiD with large input size.
Furthermore as expected, the extractive reader without reranker already has top-128 passages at the input, and improvements from the passage reranking are only negligible if any (less than 1 EM).

Finally, the results on NQ-Open and EfficientQA suggest applying the binary decision does not bring large improvements over the score aggregation if any. 
However, notice that this is not the case for TriviaQA, where the generative reader performs significantly better compared to extractive reader, suggesting both component fusions play important role in the system.

%%%%%%%%%%%%%%%%%%%%%%%%%%%%%%%%%%%%%%%%%%%%%%%%%%%%%%%%%%%%%%
%  COMPONENT FUSION TABLE (RESULTS AND ANALYSES)
%%%%%%%%%%%%%%%%%%%%%%%%%%%%%%%%%%%%%%%%%%%%%%%%%%%%%%%%%%%%%%
\begin{table}[ht]
    \centering
    \input{tables/revision/fusion-analysis_score-aggr_NQ-Open-test}
    \caption{Results for different pipeline components used for score aggregation on NQ. See text below for details.}
    \label{tab:score-aggr}
\end{table}
\begin{table}[ht]
    \centering
    \input{tables/revision/fusion-analysis_binary-decision_NQ-Open-test}
    \caption{Results for binary decision on NQ for different aggregated pipeline components from Table \ref{tab:score-aggr}.}
    \label{tab:binary-decision}
\end{table}

%%%%%%%%%%%%%%%%%%%%%%%%%%%%%%%%%%%%%%%%%%%%%%%%%%%%%%%%%%%%%%
%  COMPONENT FUSION (RESULTS AND ANALYSES)
%%%%%%%%%%%%%%%%%%%%%%%%%%%%%%%%%%%%%%%%%%%%%%%%%%%%%%%%%%%%%%
\item[Component fusion.] 
Furthermore, we analyze the performance of each component combination in the score aggregation and its impact on the component fusion via binary decision.
Both fusions are tuned on validation data and reported on test data of the NQ-Open dataset with full index. See Appendix \ref{app:additional_comp_fusion} for analysis on additional datasets.
Table \ref{tab:score-aggr} shows all relevant combinations of ranker \emph{r}, reranker \emph{rr}, extractive reader \emph{e} and generative reader \emph{g} probabilities used in score aggregation. 
In overall, we observe minor improvements up to ~1EM when combining retriever and reranker scores with reader.
The impact of adding a binary decision after the score aggregation is shown in Table \ref{tab:binary-decision}. 
Interestingly, the binary decision component significantly improves the performance only without reranked answer scores (first row in both tables), which probably corresponds to an ensemble effect. 
However, fusing the generative and extractive reader via binary decision performs significantly worse on NQ-Open than fusing both readers together with score aggregation (first row in Table \ref{tab:binary-decision} vs. last row in Table~\ref{tab:score-aggr}). As already noted in ablations, we find this to be quite the opposite for TQ-Open. We hypothesize that the binary decision is strong in cases, where generative reader performs better to extractive reader (the case of TQ-Open).

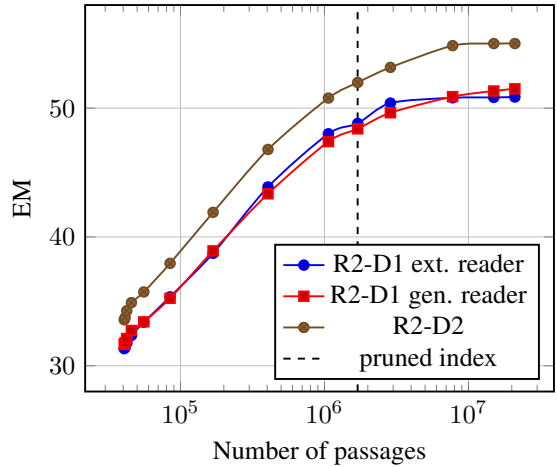
\begin{figure}
    \centering
    \scalebox{0.90}{\input{charts/index-size-analysis-NQ}}
    \caption{Index size analysis on NQ-Open test data. Note the R2-D1 system uses retriever and reranker (R2) but only one reader (D1).}
    \label{fig:nq_size_analysis}
\end{figure}

\item[Effect of index size.]
Next we analyze the effect of index size. 
We start by including all the golden passages from the training data (40,670 for NQ-Open, 50,502 for TQ-Open). 
We find the difference between using the full index and only golden passages is about 21EM (21.27 for NQ-Open, 21.01 for TQ-Open). 
Next, we consider adding the passages according to ranking produced by pruner.
We plot the system performance on NQ-Open test data as a function of index size in Figure \ref{fig:nq_size_analysis}. 
Results on TQ-Open follow the same trends and can be found in Appendix \ref{sec:trivia_size_analysis}.

\item[Connection between retrieval embeddings and pruner.]
\label{sec:results_and_analysis_embconnection}
Finally, we analyze whether the DPR embeddings capture the same phenomena as our pruner does. 
Starting with basic statistics we compute the mean and the variance vectors of $d$-dimensional embeddings representing pruned (1.7M) set of documents $P$ and those which represent the rest of the knowledge base $N$. 
We find that computing the L2 distance between mean and variance yields order of magnitude different results than the distance between randomly permuted splits of $P \cup N$ with the same size. 
Next we found the significant difference between average length of embedding vectors from $P$ and $N$. 
Conclusively, we train a logistic regression classifier on balanced dataset constructed from $P$ and subset of $N$, which predicts whether the passage belongs into the pruned set $P$ or not based on its embedding. 
We found the classifier achieves 84.1\% accuracy on the dev set, confirming our hypothesis that the apriori relevance is indeed captured in these embeddings.

\end{description}

%%%%%%%%%%%%%%%%%%%%%%%%%%%%%%%%%%%%%%%%%%%%%%%%%%%%%%%%%%%%%%
% RELATED WORK
%%%%%%%%%%%%%%%%%%%%%%%%%%%%%%%%%%%%%%%%%%%%%%%%%%%%%%%%%%%%%%
\section{Related Work}
\begin{description}[style=unboxed,leftmargin=0em,listparindent=\parindent,parsep=0pt,]
\item[Pruning the document space.]
% Earlier work in IR explored multi-stage ranking-pruning-reranking methods. \citet{tonellotto2013efficient} shown the query-conditioned pruning greatly improves the search engine response time. Furthermore, their approach is dynamic --- they postulate that easier queries have more relevant documents highly ranked within the sample, and hence can be more aggressively pruned.  \citet{kulkarni2015selective} proposed an approach for sharding the document space according to its topics. When the topic clustering step is done the system firstly chooses relevant shards for given query, and then proceeds to search their contents. 

% \citet{min2021neurips} presented a simple baseline which includes an index containing 1.65M passages. The approach finds the Wikipedia articles from top-5 positive passages from DPR training data for NQ \cite{karpukhin2020dense} (therefore golden passage, and highest-ranking BM25 passages to each question) and collects all their respective passages into the index.

\citet{min2021neurips} presented a simple baseline which includes an index containing 1.65M passages. 
These include all passages from the Wikipedia articles assigned to top-5 positive passages from DPR training data for NQ \cite{karpukhin2020dense} (therefore golden passage, and highest-ranking BM25 passages to each question). 
However this pruning approach led to -6.7 decrease in exact match, while ours led to at most -3 EM with similar amount of passages.

Similar to our work, \citet{izacard2020memory} employed three strategies to reduce the size of the index: the first is to learn a DPR encoder with embeddings projected to lower dimension, the second is to use product quantization \cite{gray1998quantization}, and the third is a linear classifier (a pruner) which filters articles based on their title and list of categories. 
Nonetheless, their pruning approach leads to \textasciitilde 4 EM loss in performance with FiD-large, while still retaining 10M passages.

\item[Dense retrieval.]
\citet{lee-etal-2019-latent} proposed the unsupervised pretraining method named inverze-cloze task. 
Fine-tuning such pretrained system via distant supervision surpassed the BM25 baseline for the first time in open-QA. 
\citet{guu2020realm} demonstrated pre-training retriever and reader from scratch using an unsupervised masked language model. 
\citet{xiong2020progressively} demonstrated a pre-training method that does not require massive computational resources for unsupervised pre-training. 
\citet{karpukhin2020dense} adopted supervised-only approach based on dual-encoder architecture, which surprisingly overtook the unsupervised approaches. 
\citet{khattab2020relevance} adopted COLBERT \cite{khattab2020colbert}, an approach introduced in IR that models fine-grained interaction between question and passage, for open-domain QA. \citet{lewis2021paq} generated a colossal corpus of 65M questions and their respective answers. 
Given a question, they showed it is possible to match the state-of-the-art performance by picking an answer of the most similar question according to the learned model. 
\citet{izacard2020distilling} demonstrated a way of distilling FiD reader knowledge into retriever, improving its retrieval significantly, while also allowing to train retriever from scratch without any passage relevance supervision.

\item[Passage reranker.]

Previous work in QA based on neural nets used bi-LSTM encoders \cite{wang2018r3,lee2018ranking} that score each document independently. Over time, bi-LSTM were replaced by BERT-like transformer encoders \cite{qiao2019understanding,wang2019multi}. 
 For document ranking, \citet{nogueira2019multistage} proposed a multi-stage architecture. The first stage scores each document independently, and the second estimates the more relevant document from all document pairs. Another document ranking approach uses the seq2seq model to generate a true or false answer to the document's relevance to the query \cite{nogueira2020document}. 
Recent works have often focused on effective reranking. \citet{xin2020early} achieved inference speedup using early exiting, \citet{jang2020document} proposed a smaller and faster model, and \citet{mao2021reader} came up with a method which uses reader's predictions to rerank the passages.
\citet{iyer2020reconsider} marked answer predictions from reader in passages and learned to re-rank top answers along with their passage context.
% \todo{cross-attention}:
% \begin{enumerate}
%     \item \url{https://arxiv.org/pdf/2010.10757.pdf}
%     \item \url{https://arxiv.org/pdf/1905.01969.pdf}
%     \item \url{https://arxiv.org/pdf/2005.00181.pdf}
% \end{enumerate}
Our reranker is most similar to \citet{nogueira2019passage, luan2020sparse}, except that unlike in IR, we assume there is just one correct passage and thus train our model via categorical cross-entropy.

\item[Reader.] 
Recent work considers two approaches towards modeling the reader --- generative and extractive. The generative reader generates an answer while conditioned on question or relevant passages \cite{roberts2020much,lewis2020retrieval}. \citet{min2020ambigqa} proposed to concatenate a question with top retriever passages as the input of pretrained seq2seq generative model. \citet{izacard2020leveraging} showed its suffices to concatenate the passages in the decoder of seq2seq model, increasing the amount of top-passages the model can depend on dramatically.
% Tu sa zuzit iba na multi-paragraph RC, to je to co robime. Cisty RC je podla mna uz priliz "unrelated".
The extractive reader used in open-QA assumes that the answer is a continuous span string in multiple paragraphs  \cite{chen2017reading}. \citet{clark-gardner-2018-simple} proposed to aggregate the probabilities of distantly supervised answer matches via maximum marginal likelihood (MML). \citet{lin2018denoising} proposed to denoise distantly supervised answer string matches in MML via paragraph-ranker. \citet{min2019discrete} introduced a learning objective, which decides randomly whether to use MML objective or hard expectation-minimization via continuous annealing scheme during the training. \citet{cheng2020probabilistic} experimented with different assumptions for MML, showing improvement when marginalizing over components of span probability independently. \citet{fajcik2020rethinking} proposed to model joint span probability directly via compound objective, instead of modeling the probability of span's start and end independently.  \citet{karpukhin2020dense} incorporated an independent passage classifier loss to his MML objective.

Unlike others, our work incorporates both, the generative and the extractive approach. While our generative reader follows \citet{izacard2020leveraging}, our extractive reader uses a novel loss function, which includes marginalizing over target passages independently of its other components.

\end{description}

\section{Conclusion}
This work proposed R2-D2, a novel state-of-the-art pipeline for open-domain QA based on 4 components: retriever, reranker, generative reader and extractive reader. Furthermore, it proposed an approach for reducing the pipeline size to fit 6GiB Docker Image. The core idea of our approach was to drastically reduce the colossal number of passages commonly used within the knowledge-base of retrieval-based open-domain QA systems (by 92\%) with only minor loss of performance (-3 EM).
We believe our pipeline composed of multiple heterogeneous components is an ideal benchmark system for future research. Additionally, the pruned index size opens up new possibilities, as it now fits to most modern GPUs. 
However, with such a drastic reduction of knowledge-base, more questions arise: \textit{What is it that makes the passage being apriori relevant? Does this strong prior over passages suggest that these  open-domain answering datasets aren't really ``open''?}. We would like to address these questions in our future research.

\section*{Acknowledgments}
We would like to thank Jan Doležal for implementing an R2-D2 demo.
This work was supported by the Czech Ministry of Education, Youth and Sports, subprogram INTERCOST, project code: LTC18006.
The computation used the infrastructure supported by the Czech Ministry of Education, Youth and Sports from the Large Infrastructures for Research, Experimental Development and Innovations project „IT4Innovations National Supercomputing Center – LM2018140“.

\bibliography{tacl2018}
\bibliographystyle{acl_natbib}

\appendix
\clearpage

\section{Additional Index Size Analysis}
This section contains additional index size analysis analogous to Figure \ref{fig:nq_size_analysis} on Trivia-Open test data shown in Figure \ref{fig:trivia_size_analysis}.
\label{sec:trivia_size_analysis}
\begin{figure}[H]
    \centering
    \scalebox{0.90}{\input{charts/index-size-analysis-Trivia}}
    \caption{Index size analysis on TQ-Open test data.}
    \label{fig:trivia_size_analysis}
\end{figure}
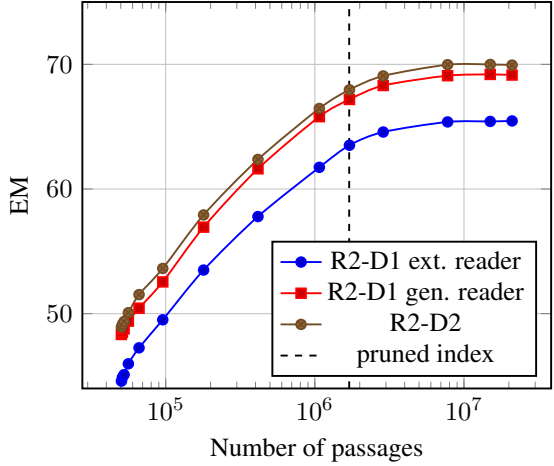

%\clearpage
%\newpage
%\vfill\eject
% \pagebreak
\section{Additional Accuracy Analysis}
\label{app:accuracy_at_k}
Analysis of Accuracy@K on EfficientQA data is shown in Figure \ref{fig:accuracy-at-k_efficientqa}, on NQ-open validation data in Figure \ref{fig:accuracy-at-k_nq-open-dev} and on TriviaQA validation data in Figure \ref{fig:accuracy-at-k_trivia-dev} and test data in Figure \ref{fig:accuracy-at-k_trivia-test}. The pruned version of our pipeline with reranker (reranked-pruned) performs better than the full version only with retriever (retrieved-full):
\begin{itemize}[style=unboxed,leftmargin=0em,listparindent=\parindent]   \setlength\parskip{0em}
    \item on NQ-open (dev) up to $K=32$ paragraphs,
    \item on NQ-open (test) up to $K=26$ paragraphs,
    \item on EfficientQA up to $K=43$ paragraphs,
    \item on TriviaQA (dev) up to $K=110$ paragraphs,
    \item on TriviaQA (test) up to $K=116$ paragraphs.
\end{itemize}

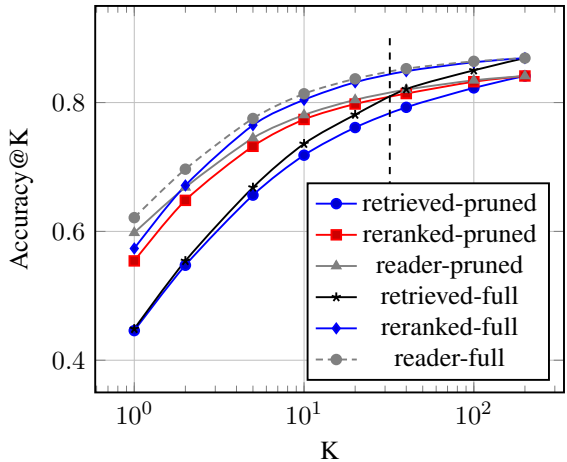
\begin{figure}[H]
    \centering
    \scalebox{0.90}{\input{charts/accuracy-at-k_NQ-Open-dev}}
    \caption{Analysis of Accuracy@K on NQ-Open (dev).}
    \label{fig:accuracy-at-k_nq-open-dev}
\end{figure}

\begin{figure}[H]
    \centering
    \scalebox{0.90}{\input{charts/accuracy-at-k_efficientqa}}
    \caption{Analysis of Accuracy@K on EfficientQA.}
    \label{fig:accuracy-at-k_efficientqa}
\end{figure}
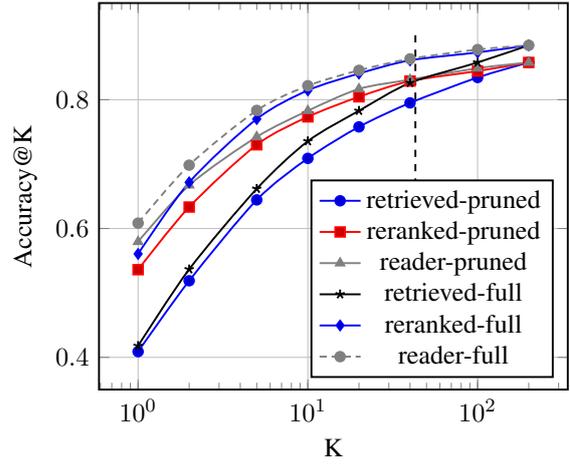

\begin{figure}[H]
    \centering
    \scalebox{0.90}{\input{charts/accuracy-at-k_trivia-dev}}
    \caption{Analysis of Accuracy@K on TriviaQA (dev).}
    \label{fig:accuracy-at-k_trivia-dev}
\end{figure}
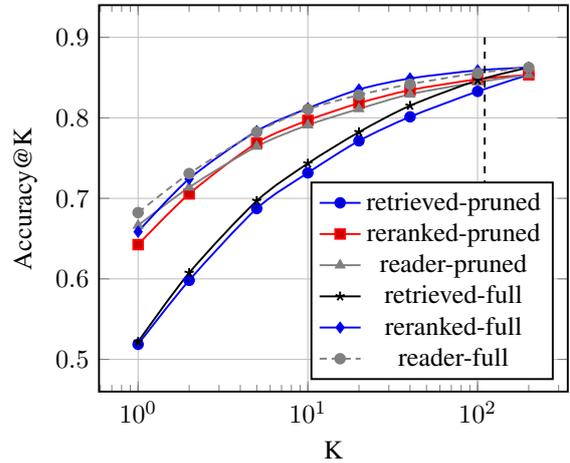

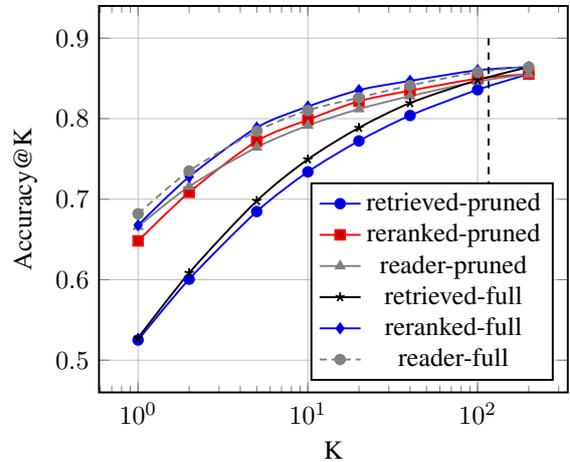
\begin{figure}[H]
    \centering
    \scalebox{0.90}{\input{charts/accuracy-at-k_trivia-test}}
    \caption{Analysis of Accuracy@K on TriviaQA (test).}
    \label{fig:accuracy-at-k_trivia-test}
\end{figure}

\section{Additional Component Fusion Analysis}
\label{app:additional_comp_fusion}
This section includes results analogical to Tables \ref{tab:score-aggr}, \ref{tab:binary-decision} on validation data of NQ-Open (Tables \ref{tab:score-aggr_nqdev}, \ref{tab:binary-decision_nqdev}), EfficientQA (Tables \ref{tab:score-aggr_effqa}, \ref{tab:binary-decision_effqa}) and TQ-Open (Tables \ref{tab:score-aggr_tqdev}, \ref{tab:binary-decision_tqdev}, \ref{tab:score-aggr_tqtest}, \ref{tab:binary-decision_tqtest}).
\begin{table}[H]
    \centering
    \scalebox{1.0}{\input{tables/revision/fusion-analysis_score-aggr_NQ-Open-dev}}
    \caption{Score aggregation -- NQ-Open (dev).}
    \label{tab:score-aggr_nqdev}
\end{table}
\begin{table}[H]
\vspace{-0.5cm}
    \centering
    \scalebox{1.0}{\input{tables/revision/fusion-analysis_binary-decision_NQ-Open-dev}}
    \caption{Binary decision -- NQ-Open (dev).}
    \label{tab:binary-decision_nqdev}
\end{table}
\begin{table}[H]
\vspace{-0.5cm}
    \centering
    \scalebox{1.0}{\input{tables/revision/fusion-analysis_score-aggr_efficientqa}}
    \caption{Score aggregation -- EfficientQA.}
    \label{tab:score-aggr_effqa}
\end{table}
\begin{table}[H]
\vspace{-0.5cm}
    \centering
    \scalebox{1.0}{\input{tables/revision/fusion-analysis_binary-decision_efficientqa}}
    \caption{Binary decision -- EfficientQA.}
    \label{tab:binary-decision_effqa}
\end{table}
\begin{table}[H]
\vspace{-0.5cm}
    \centering
    \scalebox{1.0}{\input{tables/revision/fusion-analysis_score-aggr_tq-open-dev}}
    \caption{Score aggregation -- TQ-Open (dev).}
    \label{tab:score-aggr_tqdev}
\end{table}
\begin{table}[H]
\vspace{-0.5cm}
    \centering
    \scalebox{1.0}{\input{tables/revision/fusion-analysis_binary-decision_tq-open-dev}}
    \caption{Binary decision -- TQ-Open (dev).}
    \label{tab:binary-decision_tqdev}
\end{table}
\begin{table}[H]
\vspace{-0.5cm}
    \centering
    \scalebox{1.0}{\input{tables/revision/fusion-analysis_score-aggr_tq-open-test}}
    \caption{Score aggregation -- TQ-Open (test).}
    \label{tab:score-aggr_tqtest}
\end{table}
\begin{table}[H]
\vspace{-0.5cm}
    \centering
    \scalebox{1.0}{\input{tables/revision/fusion-analysis_binary-decision_tq-open-test}}
    \caption{Binary decision -- TQ-Open (test).}
    \label{tab:binary-decision_tqtest}
\end{table}

\section{Data Pre-processing}
\label{app:data_preprocessing}
This section describes how the training datasets for reranker and extractive reader are filtered, and how the distant supervision labeling is generated. Note not each example contains golden passage, as not each example can be mapped to the used dump of Wikipedia. We use the same golden passage mapping as \citet{karpukhin2020dense}.

For passage reranking, the input must contain at least one positive example. We meet this condition either by adding a golden passage or searching for the passage with an answer in the top-400 results retrieved by DPR. In detail about the search, first the Simple tokenizer proposed in DrQA\footnote{\url{https://github.com/facebookresearch/DrQA}} tokenizes each passage and golden answer. The positive example is the best-scored tokenized passage that contains an exact match with one of the tokenized answers. Note the search proceeds in the same way as in DPR\footnote{\url{https://github.com/facebookresearch/DPR}} implementation.

The extractive reader is trained only on samples which contain exact match to at least one of the annotated answers in the top-1 passage, or golden passage if it is available. The exact match is performed on the subword token level (i.e. in ELECTRA's tokenization).

Next, the span annotations are extracted from the passages at the reader's input. Note each sample may contain multiple answers. The annotations for each answer in each sample are obtained differently in retrieved passages and in the golden passage. 
For retrieved passages, we search for the answer's exact matches in passages, and use each match as target annotation. 
For golden passage, we also search for the answer's exact matches in it. If there is none, the answer is soft-matched with single sub-sequence of golden passage, which yields highest non-zero F1 score. The F1 soft-match is also performed on the subword token level. Therefore answers with zero highest F1 soft-match with golden passage and no exact match in any of the reader's input passages are discarded.

% ====// VERSION 2 ====\\

Because the brute-force computation of a span with the greatest  nonzero F1 score is potentially very demanding, we found the length limit for spans that are worth searching (see Theorem \ref{the:F1worthSearchingTheorem}).

To compare brute-force with upper bound implementation, we run an experiment on 16,741 passages (retrieved for NQ-Open dev). The average time per passage for brute-force was 121 ms and only 9 ms for implementation with an upper bound.

The soft match is described in Algorithm \ref{alg:softMatch}. It assumes that there is no exact match.

\begin{algorithm}
    \caption{Soft match}
    \label{alg:softMatch}
    
    \begin{algorithmic}[1]
    \Require set of spans S and answer span a
    \Function{SoftMatch}{$\texttt{S}, \texttt{a}$}
        \State $\texttt{actSize} \gets 1$
        \State $\texttt{lenLimit} \gets 2$
        \State $\texttt{bestSpan} \gets \texttt{None}$
        \State $\texttt{bestScore} \gets 0$
        
        \While{$\texttt{actSize} < \texttt{lenLimit}$}
            \ForAll{$t \in S$ of size $\texttt{actSize}$} 
                \State $\texttt{score} \gets \texttt{F1}(t,a)$
                \If{$\texttt{score}>\texttt{bestScore}$}
                    \State $\texttt{bestSpan} \gets t$
                    \State $\texttt{bestScore} \gets \texttt{score}$
                    \State $\texttt{lenLimit} \gets |a|\frac{|t|+|a|-s_{ta}}{s_{ta}}$
                \EndIf
            \EndFor
            \State $\texttt{actSize} \gets \texttt{actSize}+1$
        \EndWhile
    \State \Return $\texttt{bestSpan}$
    \EndFunction
    \end{algorithmic}
\end{algorithm}

\begin{lemma}
\label{lem:upBounIsGreater}
Let $t$ and $a$ be non-empty spans and $0<s_{ta}\leq|a|$ number of shared tokens for them. Then\footnote{|x| symbolises number of tokens in span x} 
\begin{equation}
   |t| \leq |a|\frac{|t|+|a|-s_{ta}}{s_{ta}} \: .
\end{equation}
\end{lemma}
\begin{proof}

To prove it by contradiction assume that
\begin{equation}
    |t| > |a|\frac{|t|+|a|-s_{ta}}{s_{ta}} \:,
\end{equation}
then

\begin{equation}
    s_{ta}|t| > |a||t|+|a||a|-|a|s_{ta} \:,
\end{equation}
and also $0 < s_{ta} \leq |a|$, thus $|a||a|-|a|s_{ta} \geq 0$. Therefore even if we assume that\\$|a||a|-|a|s_{ta} = 0$. We get
\begin{equation}
\begin{split}
    s_{ta}|t| > |a||t| \\
    s_{ta} > |a| \: ,
\end{split}
\end{equation}

which is in contradiction with $0 < s_{ta} \leq |a|$.
\end{proof}

\begin{theorem}
\label{the:F1worthSearchingTheorem}
Let $S$ be a set of non-empty spans, $a$ an non-empty answer span, $t$ non-empty trial span, $0 < s_{ta} \leq|a|$ is number of shared tokens for $t$ and $a$, and $S_b = \{z | z \in S \land |z| \geq |a|\frac{|t|+|a|-s_{ta}}{s_{ta}}  \}$. Then the theorem states that 
\begin{equation}
    \forall x \in S_b(\operatorname{F1}(x,a) \leq \operatorname{F1}(t,a)) \: .
\end{equation}

\end{theorem}
\begin{proof}
To prove it by contradiction assume that

\begin{equation}
    \exists x \in S_b (\operatorname{F1}(x,a) > \operatorname{F1}(t,a)) \: .
\end{equation}

F1 score can be expressed as:
\begin{equation}
    \operatorname{F1(b,c)}=\frac{2s_{bc}}{|b|+|c|} \: ,
\end{equation}

thus 
\begin{equation}
\label{eq:ineqF1}
    \frac{2s_{xa}}{|x|+|a|} >  \frac{2s_{ta}}{|t|+|a|} \: .
\end{equation}

From Lemma \ref{lem:upBounIsGreater} $|t|\leq|x|$. Therefore $s_{ta}<s_{xa}$, to satisfy the inequality (in equation \ref{eq:ineqF1}), and we know that $0 < s_{xa} \leq |a|$. So let the $s_{xa}=|a|$ (the maximum) then

\begin{equation}
\begin{split}
    \frac{2|a|}{|x|+|a|} >  \frac{2s_{ta}}{|t|+|a|}
    \\
    |a|(|t|+|a|) > s_{ta}|x| + s_{ta}|a|
    \\
    |x| < |a| \frac{|t|+|a|-s_{ta}}{s_{ta}} \: ,
\end{split}
\end{equation}
which is in contradiction with $x \in S_b$.
\end{proof}

\section{Softmax Notation}
\label{app:softmax_not}
Usually, softmax function $ \sigma: \mathbb{R}^K \rightarrow \mathbb{R}^K $ is  defined as:
\begin{equation}
   \sigma(v)_i = \frac{e^{v_i}}{\sum^K_{j=1} e^{v_j}}.
\end{equation}

However, some parts of this work used variant of softmax that is defined as follows:
\begin{equation}
\begin{split}
    \operatorname*{softmax}\limits_{x \in D} \big({f(x)}\big)_{y} = \frac{e^{f(y)}}{\sum\limits_{x \in D} e^{f(x)}} \: ,
\end{split}
\end{equation}
where $D$ is the input set, $f : D \rightarrow \mathbb{R}$, $y \in D$.

\section{Decoding the Distributions from the Extractive Reader}
\label{app:decoding_ext_probs}
We analyzed the subsets of joint probability space over spans obtained via multiplication of distributions as explained in section \ref{ss:ext_reader} in Table \ref{tab:ext_r_prob_space}. The factors of this space are the distribution given by the outer product of independent probability distributions $\boldsymbol{P}_{start}(.) \boldsymbol{P}_{end}(.)^\top$ denoted as I, joint probability distribution $\boldsymbol{P}_{joint}(.)$ denoted as J, and passage distribution $\boldsymbol{P}_{passage}(.)$ denoted as C.
\begin{table}[H]
\scalebox{0.90}{
\begin{tabular}{cccc}
\toprule
\textbf{Factorization} & \multicolumn{1}{l}{\textbf{NQ-dev}} & \multicolumn{1}{l}{\textbf{NQ-test}} & \multicolumn{1}{l}{\textbf{EfficientQA}} \\ \midrule
I                                       & 48.32                              & 50.58                               & 47.33                                   \\
J                                       & 48.53                              & \textbf{51.25}                      & \textbf{47.83}                          \\
I+J                                     & \textbf{48.57}                     & 50.83                               & \textbf{47.83}                          \\
I+C                                     & 48.22                              & 50.55                               & 47.22                                   \\
J+C                                     & 48.49                              & 51.11                               & 47.56                                   \\
I+J+C                                   & 48.50                                & 50.86                               & 47.67                                   \\\bottomrule
\end{tabular}
}

    \caption{The results of extractive reader with different types of distribution used for decoding. See text for details.}
    \label{tab:ext_r_prob_space}
\end{table}

\section{Passage Reranker Revision}
In previous version of this work we used a Longformer encoder \cite{Beltagy2020Longformer} with concatenated passages at it's input to benefit from the early fusion between passages. Therefore each passage was scored not only according to the question but also according to other passages. However, we did not observe any significant benefits when we used the Longformer setup over a RoBERTa which scores each passage independently (see Table~\ref{tab:old_ablation_study}). 

\begin{table*}
    \centering
    \scalebox{0.95}{\input{tables/longformer-vs-roberta-ablation-study}}%
    \caption{Exact match comparison of Longformer (Long.) and RoBERTa (RoB.) based passage reranker.}
    \label{tab:old_ablation_study}
\end{table*}

\section{LDA Analysis of Pruned Passages}
We apply LDA \cite{blei2003latent} with 100 topics to random subset of 1M passages from knowledge base. Then we apply T-SNE \cite{van2008visualizing} and project LDA vectors into 2 dimensions. The results are shown in Figure \ref{fig:lda_analysis}. The pruned passages seem to be evenly distributed through the topics.

\begin{figure*}[t]
    \centering

    \scalebox{0.52}{\includegraphics[]{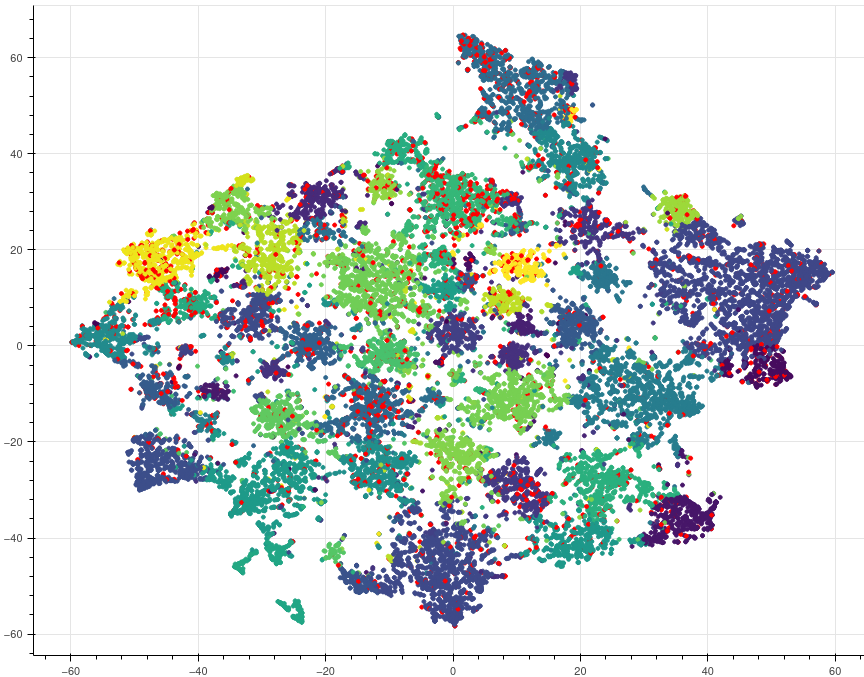}}
    \caption{T-SNE plot of Latent dirichlet allocation (LDA) trained over 1M passages of Wikipedia. Red points are pruned passages from the 1.7M set. Other colors mark passages by their most dominant topic found by LDA. We remove passages with maximum topic weight lesser than 0.2 to clusterize the plot data.}
    \label{fig:lda_analysis}
\end{figure*}

\end{document}

%% file: figures/pipeline-v3.tex
\def\col{2.00cm}
\def\row{3.30cm}
\def\blockwidth{3.1cm}
\def\blockwidthlarge{7.1cm}
\def\blockheight{0.90cm}
\def\arrowrate{0.3}

\def\outputcolor{gray!20}
\def\examplecolor{white}

\newcommand\Textbox[2]{%
    \parbox[c][\dimexpr#1-7.7pt][c]{2.3cm}{\centering{\bf #2}}}

\begin{tikzpicture}[->,line width=2pt,>=latex,node distance=\lvspace,
                    font=\footnotesize,
                    thick,
                    module/.style={
                        line width=1pt,
                        rectangle split, 
                        rectangle split parts=3,
                        rectangle split part fill={blue!30, gray!20, gray!10},
                        draw,
                        text width=\blockwidth, 
                        align=center, 
                        rounded corners=0.1cm,
                    }
                ]
    \node[align=center,text width=0.5\textwidth] at (1*\col,3.7*\row) (question) {\large In which Czech city is the brewery of its largest beer exporter?};

    \node[module,rectangle split part fill={blue!30, \outputcolor, \examplecolor}, text width=\blockwidthlarge]   at (1*\col,3*\row) (R)  {\Textbox{\blockheight}{Retriever}\nodepart{two}{top-K passages}\nodepart[align=left]{three}{1. ... town of České Budějovice, known as Budweis... \\ 2. Czech Beer Festival is the biggest ... \\ 3. Plzeň, also called Pilsen is a city...}};
    \node[cylinder, draw, ,thick,aspect=0.2,
        minimum height=0.8cm,minimum width=1.0cm,
        shape border rotate=90, align=center] at (-0.5*\col,3.2735*\row) (datastorage) {index};

    \node[module,rectangle split part fill={cyan!40, \outputcolor, \examplecolor}, text width=\blockwidthlarge]   at (1*\col,2*\row) (RR) {\Textbox{\blockheight}{Passage reranker}           \nodepart{two}{top-K reranked passages}\nodepart[align=left]{three}{1. Plzeň, also called Pilsen is a city... \\ 2. ... town of České Budějovice, known as Budweis... \\ 3. Czech Beer Festival is the biggest ...}};
    \node[module,rectangle split part fill={green!30, \outputcolor, \examplecolor}]  at (0*\col,1*\row) (ER) {\Textbox{\blockheight}{Extractive reader}          \nodepart{two}{top-M answer spans}\nodepart[align=left]{three}{1. České Budějovice  \\ 2. Festival \\ 3. Plzeň}};
    \node[module,rectangle split part fill={yellow!40, \outputcolor, \examplecolor}] at (2*\col,1*\row+2.1ex) (AR) {\Textbox{\blockheight}{Abstractive reader}         \nodepart{two}{top generated answer}\nodepart[align=left]{three}{1. Brno}};
    \node[module,rectangle split part fill={yellow!40, \outputcolor, \examplecolor}]    at (0*\col,0*\row) (ARR)  {\Textbox{\blockheight}{Abstractive reader}\nodepart{two}{top-M reranked spans}\nodepart[align=left]{three}{1. Plzeň \\ 2. Festival \\ 3. České Budějovice}};
    \node[module,rectangle split part fill={red!30, \outputcolor, \examplecolor},]    at (0*\col,-1*\row) (AGGR)  {\Textbox{\blockheight}{score aggregation}\nodepart{two}{top-M aggr. spans}\nodepart[align=left]{three}{1. Plzeň \\ 2. České Budějovice \\ 3. Festival}};
    \node[module,rectangle split part fill={red!40, \outputcolor, \examplecolor}]    at (1*\col,-2*\row+2.1ex) (BD)  {\Textbox{\blockheight}{binary decision}\nodepart{two}{top answer}\nodepart[align=left]{three}{1. Plzeň}};
    
    %\draw[->,rounded corners] (datastorage.south) to (R.north-|datastorage.south);
    \draw[->,rounded corners] (question.south) to (R.north-|question.south);
    \draw[->,rounded corners] (R) to (RR);
    \draw[->,rounded corners] (RR) to ($(RR.south)!\arrowrate!(ER.north-|RR.south)$) to ($(ER.north|-RR.south)!\arrowrate!(ER.north)$) to (ER.north);
    \draw[->,rounded corners] (RR) to ($(RR.south)!\arrowrate!(AR.north-|RR.south)$) to ($(AR.north|-RR.south)!\arrowrate!(AR.north)$) to (AR.north);
    \draw[->,rounded corners] (ER) to (ARR);
    \draw[->,rounded corners] (AR) to ($(AGGR.south-|AR.south)!\arrowrate!(BD.north-|AR.south)$) to ($(BD.north|-AGGR.south)!\arrowrate!(BD.north)+(0.5cm,0)$) to ($(BD.north)+(0.5cm,0)$);
    \draw[->,rounded corners] (RR) to ($(RR.south|-ARR.east)+(0,0.9cm)$) to ($(ARR.east)+(0,0.9cm)$);
    \draw[->,rounded corners] (RR) to ($(RR.south|-AGGR.east)+(0,0.9cm)$) to ($(AGGR.east)+(0,0.9cm)$);
    \draw[->,rounded corners] (ARR) to (AGGR);

    \draw[->,rounded corners] (AGGR) to ($(AGGR.south)!\arrowrate!(BD.north-|AGGR.south)$) to ($(BD.north|-AGGR.south)!\arrowrate!(BD.north)-(0.5cm,0)$) to ($(BD.north)-(0.5cm,0)$);
    
    \draw[->,rounded corners] (R) to ($(R.south)!\arrowrate!(RR.north-|R.south)$) to ($(R.south)!\arrowrate!(RR.north-|R.south)-(4.22cm, 0)$) to ($(AGGR)-(2.22cm,-0.7cm)$) to ($(AGGR.west)+(0,0.7cm)$);
    
    \draw[->,rounded corners] (ER) to ($(ER.south)!\arrowrate!(ARR.north-|ER.south)$) to ($(ER.south)!\arrowrate!(ARR.north-|ER.south)-(2.00cm, 0)$) to ($(AGGR)-(2.00cm,-1.1cm)$) to ($(AGGR.west)+(0, 1.1cm)$);
\end{tikzpicture}

%% file: tables/revision/dataset_statistics.tex
\def\tablepad{1ex}
\begin{tabular}{l r r r}
\toprule
    \multicolumn{1}{c}{\bf Dataset} & \multicolumn{1}{c}{\bf Train} & \multicolumn{1}{c}{\bf Dev} &     \multicolumn{1}{c}{\bf Test}  \\
    \midrule
    NQ-Open    & 79,168 & 8,757 & 3,610 \\
    \hspace{\tablepad} 
     - filt. reranker    & 71,238 & - & - \\
    \hspace{\tablepad}  
     - filt. ext. reader   & 61,755 & - & - \\
    \hspace{\tablepad} 
     - w/ golden passage  & 58,876 & 6,515  & - \\
    TQ-Open    & 78,785 & 8,837 & 11,313 \\
    \hspace{\tablepad} 
     - filt. reranker    & 69,346 & - & - \\
    \hspace{\tablepad} 
     - filt. ext. reader   & 62,332 & - & - \\
    \hspace{\tablepad} 
     - w/ golden passage    & 60,413 & 6,760 & - \\
    EfficientQA & -      & -       & 1,800 \\
    \midrule
    NQ-Golden       & \multicolumn{1}{c}{176,628} & 4,332  & 8,698 \\
    TQ-Golden       & \multicolumn{1}{c}{181,239} & 4,516  & 9,004 \\
\bottomrule
\end{tabular}

%% file: tables/revision/systems.tex
\begin{tabular}{llccc}
\cmidrule[\heavyrulewidth]{2-5}
\multicolumn{1}{c}{} & \multicolumn{1}{c}{\textbf{Method}} & \multicolumn{1}{c}{\textbf{NQ}}& \multicolumn{1}{c}{\textbf{TQ}} & \multicolumn{1}{c}{\textbf{\#$\boldsymbol{\theta}$}} \\ \cmidrule{2-5} 
\multirow{13}{*}{\rotatebox[origin=c]{90}{Extractive}}   & BM25+BERT \cite{mao2020generation}&    37.7    &    60.1                     &  110M                 \\
                     & Hard EM \cite{min2019discrete}      &    28.1   &    50.9                   &  110M                 \\
                     & Path Retriever \cite{asai2019learning}&  32.6   &    -                      &  447M                 \\ % BERT-base in retriever, BERT-large in reader + 2048*1024 layer
                     & Graph Retriever \cite{min2019knowledge}& 34.5   &    56.0                   &  110M                 \\
                     & ORQA \cite{lee-etal-2019-latent}           &    33.3  &    45.0             &  220M                 \\
                     & REALM \cite{guu2020realm}           &    40.4     &    -                    &  660M                 \\
                     & ProQA \cite{xiong2020progressively} &    34.3     &    -                    &  220M                 \\
                     & DPR \cite{karpukhin2020dense}       &    41.5     &    56.8                 &  220M                 \\
                     & DPR-subset$^*$ \cite{min2021neurips}    &    34.8   &    -                      &  220M                 \\
                     & RDR \cite{yang2020retriever}        &    42.1    &   57.0                  &  110M                 \\
                     & GAR+DPR \cite{mao2020generation}    &    43.8    &    -                     &  626M                 \\ % GAR contains BART-large = 406
                     & ColBERT (large) \cite{khattab2020relevance}& 48.2  &    63.2$^{-}$                &  440M                 \\ 
                     & RIDER (GAR+DPR) \cite{mao2021reader} &    48.3  &    -                       &  626M                 \\ \cmidrule{2-5}  % BERT-base & DPR & BART-large

\multirow{8}{*}{\rotatebox[origin=c]{90}{Generative}}      & BM25+SSG \cite{mao2020generation}&    35.3     &    58.6                    &  406M \\
                     & T51.1+SSM \cite{roberts2020much}    &    35.2   &    61.6                      &  11B                  \\
                     & RAG \cite{lewis2020retrieval}       &    44.5   &    56.8                      &  516M                 \\ % DPR + BART_large
                     & DPR+SSG \cite{min2020ambigqa}       &    42.2   &    -                      &  516M                 \\ % DPR + BART_large
                     & FiD-base \cite{izacard2020leveraging}&   48.2   &    65.0                      &  333M                 \\ % DPR + 222,884,352
                     & FiD-large \cite{izacard2020leveraging}&  51.4   &    67.6                      &  848M                 \\ % DPR + 737,642,496
                     & FiD-large++$^*$ \cite{izacard2020memory}&    53.6   &    71.3                   &  848M                 \\
                     & FiD-large++ \cite{izacard2020memory}&    54.7   &    \textbf{73.3}                     &  848M                 \\\cmidrule{2-5} 
\multirow{4}{*}{\rotatebox[origin=c]{90}{Ours}}            & R1-D1 (Generative)              &    49.9        &    65.4                     &  848M                 \\
                     & R1-D1 (Extractive)                  &    50.8     &    65.0                    &  445M                 \\ 
                     & R2-D2 (1.7M)                        &    52.6      &    68.0                   &  1.29B                     \\ %110M + 110M+ 335M+738M
                     & R2-D2 (21M)                         &    \textbf{55.0}   &    69.9             &  1.29B                     \\
                     & R2-D2 (21M) w/ HN-DPR               &    \textbf{55.9}   &    -             &  1.29B                     \\ \cmidrule[\heavyrulewidth]{2-5} 
\end{tabular}

%% file: tables/revision/ablation-study-wo-trivia-dev.tex
\begin{tabular}{ccc|rrr|rrr|rrr|rrr}
\toprule
\textbf{Passage}   & \multirow{2}{*}{\textbf{Readers}} & \multirow{2}{*}{\textbf{Fusion}} & \multicolumn{3}{c|}{\textbf{NQ-Open (dev)}}                                        & \multicolumn{3}{c|}{\textbf{NQ-Open}}                                              & \multicolumn{3}{c|}{\textbf{TriviaQA}}                                             & \multicolumn{3}{c}{\textbf{EfficientQA}}                                         \\
\textbf{reranking} &                                   &                                  & \multicolumn{1}{c}{1.7M} & \multicolumn{1}{c}{21M} & \multicolumn{1}{c|}{$\Delta$} & \multicolumn{1}{c}{1.7M} & \multicolumn{1}{c}{21M} & \multicolumn{1}{c|}{$\Delta$} & \multicolumn{1}{c}{1.7M} & \multicolumn{1}{c}{21M} & \multicolumn{1}{c|}{$\Delta$} & \multicolumn{1}{c}{1.7M} & \multicolumn{1}{c}{21M} & \multicolumn{1}{c}{$\Delta$} \\ \midrule
-                  & ext                               & -                                & 46.53                    & 48.12                   & -1.59                         & 48.64                    & 50.78                   & -2.14                         & 63.18                    & 65.01                   & -1.83                         & 44.50                    & 47.00                   & -2.50                        \\
-                  & gen                               & -                                & 46.23                    & 48.30                   & -2.07                         & 48.39                    & 49.92                   & -1.53                         & 63.72                    & 65.38                   & -1.66                         & 43.50                    & 44.83                   & -1.33                        \\
-                  & ext+gen                           & naive                            & 47.05                    & 49.14                   & -2.09                         & 50.00                    & 51.88                   & -1.88                         & 64.04                    & 66.17                   & -2.13                         & 45.61                    & 47.06                   & -1.45                        \\
-                  & ext+gen                           & aggr                             & 49.05                    & 50.87                   & -1.82                         & 51.94                    & 54.13                   & -2.19                         & 65.41                    & 67.42                   & -2.01                         & 47.50                    & 50.44                   & -2.94                        \\
-                  & ext+gen                           & aggr+bd                          & 49.35                    & 51.18                   & -1.83                         & 51.88                    & 54.07                   & -2.19                         & 65.69                    & 67.37                   & -1.68                         & 47.28                    & 49.72                   & -2.44                        \\ \midrule
\checkmark         & ext                               & -                                & 46.64                    & 48.38                   & -1.74                         & 48.92                    & 50.72                   & -1.80                         & 63.51                    & 65.46                   & -1.95                         & 45.06                    & 47.56                   & -2.50                        \\
\checkmark         & gen                               & -                                & 47.11                    & 49.40                   & -2.29                         & 48.31                    & 50.69                   & -2.38                         & {\ul 67.18}              & {\ul 69.14}             & -1.96                         & 45.22                    & 47.33                   & -2.11                        \\
\checkmark         & ext+gen                           & naive                            & 47.78                    & 49.99                   & -2.21                         & 50.33                    & 52.44                   & -2.11                         & 66.02                    & 68.01                   & -1.99                         & 46.78                    & 49.11                   & -2.33                        \\
\checkmark         & ext+gen                           & aggr                             & 49.89                    & 51.80                   & -1.91                         & 52.38                    & 54.90                   & -2.52                         & 66.86                    & 68.66                   & -1.80                         & \textbf{49.44}           & 52.00                   & -2.56                        \\
\checkmark         & ext+gen                           & aggr+bd                          & \textbf{50.25}           & \textbf{52.07}          & -1.82                         & \textbf{52.58}           & \textbf{54.99}          & -2.41                         & \textbf{67.96}           & \textbf{69.94}          & -1.98                         & 49.22                    & \textbf{52.22}          & -3.00                        \\ \bottomrule
\end{tabular}

%% file: charts/accuracy-at-k_nq-open-test.tex
\begin{tikzpicture}
    \begin{axis}[
            smooth,
            thick,
            %scale only axis,
            xmode=log,
            %no marks,
            legend pos=south east,
            xlabel=K,
            ylabel=Accuracy@K,
            skip coords between index={2}{4},
            skip coords between index={5}{9},
            skip coords between index={10}{19},
            skip coords between index={20}{39},
            skip coords between index={40}{99},
            skip coords between index={100}{199},
            %cycle list name=black white,
            xmajorgrids,
            ymajorgrids,
        ]
        \addplot table [x=k, y=r1-p, col sep=comma] {data/accuracy-at-k_nq-open-test.csv};
        \addplot table [x=k, y=r2-p, col sep=comma] {data/accuracy-at-k_nq-open-test.csv};
        \addplot+[mark=triangle*, color=gray, mark options={solid,fill=gray}] table [x=topk, y=selection, col sep=comma] {data/accuracy-at-k_reader-pruned_nq-open-test.csv};
        \addplot table [x=k, y=r1-f, col sep=comma] {data/accuracy-at-k_nq-open-test.csv};
        \addplot table [x=k, y=r2-f, col sep=comma] {data/accuracy-at-k_nq-open-test.csv};
        \addplot+[color=gray, mark options={solid,fill=gray}] table [x=topk, y=selection, col sep=comma] {data/accuracy-at-k_reader-full_nq-open-test.csv};

        \addplot[dashed,thick, samples=50, smooth,domain=0:6,black] coordinates {(26,0.4)(26,0.9)};
        \legend{
            retrieved-pruned, 
            reranked-pruned, 
            reader-pruned, 
            retrieved-full, 
            reranked-full, 
            reader-full
        };
    \end{axis}%
\end{tikzpicture}%

%% file: charts/component-size-v2.tex
\begin{tikzpicture}[line cap=round,font=\normalsize]
\begin{axis}[
    width=8cm, % changed size a bit
    height=7cm,
    xbar,
    xmin=0.01,
    xmax=70, 
    xlabel = Size in GiB,
    axis x line=bottom,
    %xmode=log,
    log origin=infty,
    enlarge y limits=0.01,
    enlargelimits=false,
    bar width=18pt,
    %bar shift=0pt,
    axis y line=left,
    enlarge y limits=0.10, 
    ytick = data,
    xticklabel style={
      /pgf/number format/fixed
    },
    xmajorgrids=true,
    y tick label style={text width=48pt,align=center,rotate=0},
    symbolic y coords={
        Parameters, 
        Dense Index, 
        Binaries, 
        Corpus, 
        Codes
    },
    nodes near coords,
    nodes near coords align=horizontal,
    legend style={
      cells={anchor=west},
      legend pos=north east,
    },
    reverse legend
 ]
    \addplot[xbar, bar width=10pt, fill=black!15!green, postaction={pattern=north east lines}] coordinates {
        (0.02222442627,Codes)
        (14.04961014,Corpus)
        (1.81256485,Binaries)
        (60.16888531,Dense Index)
        (4.958679843,Parameters)
     };
    \addplot[xbar, bar width=10pt, fill=white!15!blue] coordinates { 
        (0.0222244262695313,Codes) 
        (0.398887927643955,Corpus) 
        (1.01264190673828,Binaries)
        (2.24073219113052,Dense Index)
        (2.28196194581687,Parameters)
    };
    \legend{full system, pruned system}
\end{axis}
\end{tikzpicture}%

%% file: tables/fusion-analysis_score-aggr_nq-open-test.tex
\begin{tabular}{c|rrrr}
\toprule
$\boldsymbol{P}_*$ & \multicolumn{1}{c}{$\emptyset$} & \multicolumn{1}{c}{$\{r\}$} & \multicolumn{1}{c}{$\{{rr}\}$} & \multicolumn{1}{c}{$\{{r},{rr}\}$} \\ 
\midrule
$\{e\}$                                 & 50.86                           & 51.41                       & 51.80                      & 51.27                              \\
$\{g\}$                                 & 53.43                           & 53.68                       & 53.52                      & 53.63                              \\
$\{e, g\}$                              & 54.96                           & \bf 55.07                   & \bf 55.07                  & 54.96                            \\
\bottomrule
\end{tabular}

%% file: tables/fusion-analysis_binary-decision_nq-open-test.tex
\begin{tabular}{c|rrrr}
\toprule
$\boldsymbol{P}_*$ & \multicolumn{1}{c}{$\emptyset$} & \multicolumn{1}{c}{$\{r\}$} & \multicolumn{1}{c}{$\{{rr}\}$} & \multicolumn{1}{c}{$\{{r},{rr}\}$} \\ 
\midrule
$\{e\}$                                 & 53.38                           & 53.60                       & 53.49                          & 53.35                              \\
$\{g\}$                                 & 53.57                           & 53.49                       & 53.57                          & 53.55                              \\
$\{e, g\}$                              & 55.01                           & \bf 55.18                   & 54.96                          & 55.01                              \\
\bottomrule
\end{tabular}

%% file: charts/index-size-analysis-NQ.tex
\begin{tikzpicture}
    \begin{axis}[
            smooth,
            thick,
            xmode=log,
            legend pos=south east,
            xlabel=Number of passages,
            ylabel=EM,
            xmajorgrids,
            ymajorgrids,
            ymin=28,
            ymax=58
        ]
        %\addplot table [x=p_class, y=r2d2++, col sep=comma] {data/revision/index-size-analysis.csv};
        
        \addplot table [x=p_total, y=r2d1_ext, col sep=comma] {data/index-size-analysis-NQ.csv};
        \addplot table [x=p_total, y=r2d1_gen, col sep=comma] {data/index-size-analysis-NQ.csv};
        \addplot table [x=p_total, y=r2d2++, col sep=comma] {data/index-size-analysis-NQ.csv};
        \addplot[dashed,thick, samples=50, smooth,domain=0:6,black] coordinates {(1700000,28)(1700000,58)};
        \legend{R2-D1 ext. reader, R2-D1 gen. reader, R2-D2, pruned index};

    \end{axis}%
\end{tikzpicture}%

%% file: charts/index-size-analysis-Trivia.tex
\begin{tikzpicture}
    \begin{axis}[
            smooth,
            thick,
            xmode=log,
            legend pos=south east,
            xlabel=Number of passages,
            ylabel=EM,
            xmajorgrids,
            ymajorgrids,
            ymin=44,
            ymax=75
        ]
        %\addplot table [x=p_class, y=r2d2++, col sep=comma] {data/revision/index-size-analysis.csv};
        
        \addplot table [x=p_total, y=r2d1_ext, col sep=comma] {data/index-size-analysis-Trivia.csv};
        \addplot table [x=p_total, y=r2d1_gen, col sep=comma] {data/index-size-analysis-Trivia.csv};
        \addplot table [x=p_total, y=r2d2++, col sep=comma] {data/index-size-analysis-Trivia.csv};
        \addplot[dashed,thick, samples=50, smooth,domain=0:6,black] coordinates {(1700000,44)(1700000,75)};
        \legend{R2-D1 ext. reader, R2-D1 gen. reader, R2-D2, pruned index};

    \end{axis}%
\end{tikzpicture}%

%% file: charts/accuracy-at-k_nq-open-dev.tex
\begin{tikzpicture}
    \begin{axis}[
            smooth,
            thick,
            %scale only axis,
            xmode=log,
            %no marks,
            legend pos=south east,
            xlabel=K,
            ylabel=Accuracy@K,
            skip coords between index={2}{4},
            skip coords between index={5}{9},
            skip coords between index={10}{19},
            skip coords between index={20}{39},
            skip coords between index={40}{99},
            skip coords between index={100}{199},
            %cycle list name=black white,
            xmajorgrids,
            ymajorgrids,
        ]
        \addplot table [x=k, y=r1-p, col sep=comma] {data/accuracy-at-k_nq-open-dev.csv};
        \addplot table [x=k, y=r2-p, col sep=comma] {data/accuracy-at-k_nq-open-dev.csv};        \addplot+[mark=triangle*, color=gray, mark options={solid,fill=gray}] table [x=topk, y=selection, col sep=comma] {data/accuracy-at-k_reader-pruned_nq-open-dev.csv};
        \addplot table [x=k, y=r1-f, col sep=comma] {data/accuracy-at-k_nq-open-dev.csv};
        \addplot table [x=k, y=r2-f, col sep=comma] {data/accuracy-at-k_nq-open-dev.csv};
        \addplot+[color=gray, mark options={solid,fill=gray}] table [x=topk, y=selection, col sep=comma] {data/accuracy-at-k_reader-full_nq-open-dev.csv};

        \addplot[dashed,thick, samples=50, smooth,domain=0:6,black] coordinates {(32,0.4)(32,0.9)};
        \legend{
            retrieved-pruned, 
            reranked-pruned, 
            reader-pruned, 
            retrieved-full, 
            reranked-full, 
            reader-full
        };
    \end{axis}%
\end{tikzpicture}%

%% file: charts/accuracy-at-k_efficientqa.tex
\begin{tikzpicture}
    \begin{axis}[
            smooth,
            thick,
            %scale only axis,
            xmode=log,
            %no marks,
            legend pos=south east,
            xlabel=K,
            ylabel=Accuracy@K,
            skip coords between index={2}{4},
            skip coords between index={5}{9},
            skip coords between index={10}{19},
            skip coords between index={20}{39},
            skip coords between index={40}{99},
            skip coords between index={100}{199},
            %cycle list name=black white,
            xmajorgrids,
            ymajorgrids,
        ]
        \addplot table [x=k, y=r1-p, col sep=comma] {data/accuracy-at-k_efficientqa.csv};
        \addplot table [x=k, y=r2-p, col sep=comma] {data/accuracy-at-k_efficientqa.csv};
        \addplot+[mark=triangle*, color=gray, mark options={solid,fill=gray}] table [x=topk, y=selection, col sep=comma] {data/accuracy-at-k_reader-pruned_efficientqa.csv};
        \addplot table [x=k, y=r1-f, col sep=comma] {data/accuracy-at-k_efficientqa.csv};
        \addplot table [x=k, y=r2-f, col sep=comma] {data/accuracy-at-k_efficientqa.csv};
        \addplot+[color=gray, mark options={solid,fill=gray}]  table [x=topk, y=selection, col sep=comma] {data/accuracy-at-k_reader-full_efficientqa.csv};
                
        \addplot[dashed,thick, samples=50, smooth,domain=0:6,black] coordinates {(43,0.4)(43,0.9)};
        \legend{
            retrieved-pruned, 
            reranked-pruned, 
            reader-pruned, 
            retrieved-full, 
            reranked-full, 
            reader-full
        };
    \end{axis}%
\end{tikzpicture}%

%% file: charts/accuracy-at-k_trivia-dev.tex
\begin{tikzpicture}
    \begin{axis}[
            smooth,
            thick,
            %scale only axis,
            xmode=log,
            %no marks,
            legend pos=south east,
            xlabel=K,
            ylabel=Accuracy@K,
            skip coords between index={2}{4},
            skip coords between index={5}{9},
            skip coords between index={10}{19},
            skip coords between index={20}{39},
            skip coords between index={40}{99},
            skip coords between index={100}{199},
            %cycle list name=black white,
            xmajorgrids,
            ymajorgrids,
        ]
        \addplot table [x=k, y=r1-p, col sep=comma] {data/accuracy-at-k_trivia-dev.csv};
        \addplot table [x=k, y=r2-p, col sep=comma] {data/accuracy-at-k_trivia-dev.csv};
        \addplot+[mark=triangle*, color=gray, mark options={solid,fill=gray}] table [x=topk, y=selection, col sep=comma] {data/accuracy-at-k_reader-pruned_trivia-dev.csv};
        \addplot table [x=k, y=r1-f, col sep=comma] {data/accuracy-at-k_trivia-dev.csv};
        \addplot table [x=k, y=r2-f, col sep=comma] {data/accuracy-at-k_trivia-dev.csv};
        \addplot+[color=gray, mark options={solid,fill=gray}] table [x=topk, y=selection, col sep=comma] {data/accuracy-at-k_reader-full_trivia-dev.csv};
        
        \addplot[dashed,thick, samples=50, smooth,domain=0:6,black] coordinates {(110,0.5)(110,0.9)};
        \legend{
            retrieved-pruned, 
            reranked-pruned, 
            reader-pruned, 
            retrieved-full, 
            reranked-full, 
            reader-full
        };
    \end{axis}%
\end{tikzpicture}%

%% file: charts/accuracy-at-k_trivia-test.tex
\begin{tikzpicture}
    \begin{axis}[
            smooth,
            thick,
            %scale only axis,
            xmode=log,
            %no marks,
            legend pos=south east,
            xlabel=K,
            ylabel=Accuracy@K,
            skip coords between index={2}{4},
            skip coords between index={5}{9},
            skip coords between index={10}{19},
            skip coords between index={20}{39},
            skip coords between index={40}{99},
            skip coords between index={100}{199},
            %cycle list name=black white,
            xmajorgrids,
            ymajorgrids,
        ]
        \addplot table [x=k, y=r1-p, col sep=comma] {data/accuracy-at-k_trivia-test.csv};
        \addplot table [x=k, y=r2-p, col sep=comma] {data/accuracy-at-k_trivia-test.csv};
        \addplot+[mark=triangle*, color=gray, mark options={solid,fill=gray}] table [x=topk, y=selection, col sep=comma] {data/accuracy-at-k_reader-pruned_trivia-test.csv};
        \addplot table [x=k, y=r1-f, col sep=comma] {data/accuracy-at-k_trivia-test.csv};
        \addplot table [x=k, y=r2-f, col sep=comma] {data/accuracy-at-k_trivia-test.csv};
        \addplot+[color=gray, mark options={solid,fill=gray}] table [x=topk, y=selection, col sep=comma] {data/accuracy-at-k_reader-full_trivia-test.csv};
        \addplot[dashed,thick, samples=50, smooth,domain=0:6,black] coordinates {(116,0.5)(116,0.9)};
        \legend{
            retrieved-pruned, 
            reranked-pruned, 
            reader-pruned, 
            retrieved-full, 
            reranked-full, 
            reader-full
        };
    \end{axis}%
\end{tikzpicture}%

%% file: tables/fusion-analysis_score-aggr_nq-open-dev.tex
\begin{tabular}{c|rrrr}
\toprule
$\boldsymbol{P}_*$ & \multicolumn{1}{c}{$\emptyset$} & \multicolumn{1}{c}{$\{r\}$} & \multicolumn{1}{c}{$\{{rr}\}$} & \multicolumn{1}{c}{$\{{r},{rr}\}$} \\ \midrule
$\{e\}$                                 & 48.50                           & 49.06                       & 48.77                          & 49.27                              \\
$\{g\}$                                  & 49.91                           & 50.51                       & 50.26                          & 50.53                              \\
$\{e, g\}$                              & 51.95                           & 52.18                       & 52.02                          & \bf 52.10                             \\
\bottomrule
\end{tabular}

%% file: tables/fusion-analysis_binary-decision_nq-open-dev.tex
\begin{tabular}{c|rrrr}
\toprule
$\boldsymbol{P}_*$ & \multicolumn{1}{c}{$\emptyset$} & \multicolumn{1}{c}{$\{r\}$} & \multicolumn{1}{c}{$\{rr\}$} & \multicolumn{1}{c}{$\{r,rr\}$} \\ \midrule
$\{e\}$            & 50.71                           & 51.20                       & 50.89                        & 51.31                          \\
$\{g\}$            & 50.18                           & 50.81                       & 50.54                        & 50.84                          \\
$\{e,g\}$          & 52.31                           & \bf 52.46                       & 52.35                        & 52.40                          \\ \bottomrule
\end{tabular}

%% file: tables/revision/fusion-analysis_score-aggr_efficientqa.tex
\begin{tabular}{crrrr}
\toprule
$\boldsymbol{P}_*$ & \multicolumn{1}{c}{$\emptyset$} & \multicolumn{1}{c}{$\{r\}$} & \multicolumn{1}{c}{$\{rr\}$} & \multicolumn{1}{c}{$\{r,rr\}$} \\ \midrule
$\{e\}$   & 47.56 & 48.33 & 48.89 & 48.72          \\
$\{g\}$   & 49.11 & 49.56 & 50.22 & 50.11          \\
$\{e,g\}$ & 50.78 & 51.67 & 50.89 & \textbf{52.00} \\ \bottomrule
\end{tabular}

%% file: tables/revision/fusion-analysis_binary-decision_efficientqa.tex
\begin{tabular}{crrrr}
\toprule
$\boldsymbol{P}_*$ & \multicolumn{1}{c}{$\emptyset$} & \multicolumn{1}{c}{$\{r\}$} & \multicolumn{1}{c}{$\{rr\}$} & \multicolumn{1}{c}{$\{r,rr\}$} \\ \midrule
$\{e\}$   & 48.33 & 50.06 & 49.39 & 49.67          \\
$\{g\}$   & 48.94 & 49.50 & 50.06 & 49.72          \\
$\{e,g\}$ & 50.78 & 51.83 & 50.94 & \textbf{52.22} \\ \bottomrule
\end{tabular}

%% file: tables/revision/fusion-analysis_score-aggr_tq-open-dev.tex
\begin{tabular}{crrrr}
\toprule
$\boldsymbol{P}_*$ & \multicolumn{1}{c}{$\emptyset$} & \multicolumn{1}{c}{$\{r\}$} & \multicolumn{1}{c}{$\{rr\}$} & \multicolumn{1}{c}{$\{r,rr\}$} \\ \midrule
$\{e\}$   & 65.07 & 65.21 & 65.16 & 65.24 \\
$\{g\}$   & 67.68 & 67.72 & 67.73 & 67.76 \\
$\{e,g\}$ & 68.13 & \textbf{68.19} & 68.17 & 68.12 \\ \bottomrule
\end{tabular}

%% file: tables/revision/fusion-analysis_binary-decision_tq-open-dev.tex
\begin{tabular}{crrrr}
\toprule
$\boldsymbol{P}_*$ & \multicolumn{1}{c}{$\emptyset$} & \multicolumn{1}{c}{$\{r\}$} & \multicolumn{1}{c}{$\{rr\}$} & \multicolumn{1}{c}{$\{r,rr\}$} \\ \midrule
$\{e\}$   & 69.03 & 69.03 & 69.01 & 68.99 \\
$\{g\}$   & 69.54 & 69.46 & 69.62 & 69.70 \\
$\{e,g\}$ & 69.77 & \textbf{69.79} & 69.67 & 69.61 \\ \bottomrule
\end{tabular}

%% file: tables/revision/fusion-analysis_score-aggr_tq-open-test.tex
\begin{tabular}{crrrr}
\toprule
$\boldsymbol{P}_*$ & \multicolumn{1}{c}{$\emptyset$} & \multicolumn{1}{c}{$\{r\}$} & \multicolumn{1}{c}{$\{rr\}$} & \multicolumn{1}{c}{$\{r,rr\}$} \\ \midrule
$\{e\}$   & 65.54 & 65.64 & 65.60 & 65.61 \\
$\{g\}$   & 68.25 & 68.17 & 68.21 & 68.26 \\
$\{e,g\}$ & 68.45 & 68.57 & \textbf{68.66} & \textbf{68.66} \\ \bottomrule
\end{tabular}

%% file: tables/revision/fusion-analysis_binary-decision_tq-open-test.tex
\begin{tabular}{crrrr}
\toprule
$\boldsymbol{P}_*$ & \multicolumn{1}{c}{$\emptyset$} & \multicolumn{1}{c}{$\{r\}$} & \multicolumn{1}{c}{$\{rr\}$} & \multicolumn{1}{c}{$\{r,rr\}$} \\ \midrule
$\{e\}$   & 69.34 & 69.28 & 69.23 & 69.26 \\
$\{g\}$   & 69.76 & 69.71 & 69.65 & 69.77 \\
$\{e,g\}$ & 69.80 & 69.89 & 69.88 & \textbf{69.94} \\ \bottomrule
\end{tabular}

%% file: tables/longformer-vs-roberta-ablation-study.tex
% Please add the following required packages to your document preamble:
\begin{tabular}{ccc|rrr|rrr|rrr}
\toprule
\multirow{2}{*}{\bf Index} & \multirow{2}{*}{\bf Readers} & \multirow{2}{*}{\bf Fusion} & \multicolumn{3}{c|}{\bf NQ-Open (dev)}                              & \multicolumn{3}{c|}{\bf NQ-Open (test)}                                                 & \multicolumn{3}{c}{\bf EfficientQA}                                             \\
                           &                              &                             & Long. & \multicolumn{1}{c}{RoB.} & \multicolumn{1}{c|}{$\Delta$} & \multicolumn{1}{c}{Long.} & \multicolumn{1}{c}{RoB.} & \multicolumn{1}{c|}{$\Delta$} & \multicolumn{1}{c}{Long.} & \multicolumn{1}{c}{RoB.} & \multicolumn{1}{c}{$\Delta$} \\ \midrule
\multirow{5}{*}{1.7M}      & ext                          & -                           & 46.88 & 46.64                       & -0.24                         & 48.81                     & 48.92                       & 0.11                          & 45.22                     & 45.06                       & -0.16                        \\
                           & gen                          & -                           & 47.12 & 47.11                       & -0.01                         & 48.39                     & 48.31                       & -0.08                         & 45.56                     & 45.22                       & -0.34                        \\
                           & ext+gen                      & naive                       & 47.71 & 47.78                       & 0.07                          & 50.55                     & 50.33                       & -0.22                         & 46.94                     & 46.78                       & -0.16                        \\
                           & ext+gen                      & aggr                        & 50.17 & 49.89                       & -0.28                         & 52.11                     & 52.38                       & 0.27                          & 49.06                     & 49.44                       & 0.38                         \\
                           & ext+gen                      & aggr+bd                     & 50.50 & 50.25                       & -0.25                         & 51.99                     & 52.58                       & 0.59                          & 48.61                     & 49.22                       & 0.61                         \\ \midrule
\multirow{5}{*}{21M}       & ext                          & -                           & 48.50 & 48.38                       & -0.12                         & 50.86                     & 50.72                       & -0.14                         & 47.67                     & 47.56                       & -0.11                        \\
                           & gen                          & -                           & 49.34 & 49.40                       & 0.06                          & 51.50                     & 50.69                       & -0.81                         & 47.33                     & 47.33                       & 0.00                         \\
                           & ext+gen                      & naive                       & 49.91 & 49.99                       & 0.08                          & 53.43                     & 52.44                       & -0.99                         & 49.06                     & 49.11                       & 0.05                         \\
                           & ext+gen                      & aggr                        & 52.05 & 51.80                       & -0.25                         & 54.96                     & 54.90                       & -0.06                         & 51.56                     & 52.00                       & 0.44                         \\
                           & ext+gen                      & aggr+bd                     & 52.36 & 52.07                       & -0.29                         & 55.01                     & 54.99                       & -0.02                         & 51.06                     & 52.22                       & 1.16                         \\ \bottomrule
\end{tabular}